\documentclass[10pt,journal,compsoc]{IEEEtran}

\usepackage{amsmath, amssymb, amsfonts, amsthm, bbm}
\usepackage{booktabs,caption}
\usepackage{multirow}
\newtheorem{lemma}{Lemma}

\newtheorem{theorem}{Theorem}
\newtheorem{definition}{Definition}
\newtheorem{problem}{Problem}
\newtheorem{remark}{Remark}

\usepackage{xcolor}
\usepackage[colorlinks]{hyperref}
\hypersetup{colorlinks,breaklinks,linkcolor=red,urlcolor=blue,anchorcolor=blue,citecolor=blue}
\usepackage{graphicx}
\usepackage{subfigure}
\usepackage{algorithm, algorithmic}

\DeclareMathOperator*{\argmin}{arg\,min}

\ifCLASSINFOpdf
\else
\fi
\begin{document}
%
\title{Laplacian Convolutional Representation for Traffic Time Series Imputation}
%
%
%

\author{Xinyu~Chen, 
Zhanhong~Cheng, 
HanQin Cai,
Nicolas~Saunier, 
Lijun~Sun,~\IEEEmembership{Senior~Member,~IEEE}
\IEEEcompsocitemizethanks{
\IEEEcompsocthanksitem X. Chen and N. Saunier are with the Civil, Geological and Mining Engineering Department, Polytechnique Montreal, Montreal, QC H3T 1J4, Canada. E-mail: chenxy346@gmail.com (X. Chen), nicolas.saunier@polymtl.ca (N. Saunier).
\IEEEcompsocthanksitem Z. Cheng and L. Sun are with the Department of Civil Engineering, McGill University, Montreal, QC H3A 0C3, Canada. E-mail: zhanhong.cheng@mail.mcgill.ca (Z. Cheng), lijun.sun@mcgill.ca (L. Sun).
\IEEEcompsocthanksitem H. Cai is with the Department of Statistics and Data Science and Department of Computer Science,
University of Central Florida,
Orlando, FL 32816, USA. E-mail: hqcai@ucf.edu}
\thanks{(Corresponding author: Nicolas Saunier)}
\thanks{}}

\IEEEtitleabstractindextext{%
\begin{abstract}
Spatiotemporal traffic data imputation is of great significance in intelligent transportation systems and data-driven decision-making processes. To perform efficient learning and accurate reconstruction from partially observed traffic data, we assert the importance of characterizing both global and local trends in time series. In the literature, substantial works have demonstrated the effectiveness of utilizing the low-rank property of traffic data by matrix/tensor completion models. In this study, we first introduce a Laplacian kernel to temporal regularization for characterizing local trends in traffic time series, which can be formulated as a circular convolution. Then, we develop a low-rank Laplacian convolutional representation (LCR) model by putting the circulant matrix nuclear norm and the Laplacian kernelized temporal regularization together, which is proved to meet a unified framework that has a fast Fourier transform (FFT) solution in log-linear time complexity. Through extensive experiments on several traffic datasets, we demonstrate the superiority of LCR over several baseline models for imputing traffic time series of various time series behaviors (e.g., data noises and strong/weak periodicity) and reconstructing sparse speed fields of vehicular traffic flow. The proposed LCR model is also an efficient solution to large-scale traffic data imputation over the existing imputation models.
\end{abstract}

}

\maketitle


%
\IEEEpeerreviewmaketitle

\section{Introduction}
%
%
%
%
\IEEEPARstart{M}{issing} data imputation is a fundamental component to a wide range of applications in intelligent transportation systems (ITS), including route planning, travel time estimation, and traffic flow forecasting. Typically, traffic data can be collected by sensors (e.g., loop detectors and video cameras) on a continuous basis, producing a sequence of traffic flow time series such as speed and volume measurements. However, the real-world ITS often suffers from various operational issues such as sensor failure and network communication disorder, leading to data corruption and sparsity. Making accurate recovery of these data is vital for supporting ITS, but it still demands appropriate imputation approaches.


The basic modeling idea of missing data imputation in traffic time series is to exploit complicated spatial and temporal correlations/dependencies from partial observations, consequently leading to extensive data-driven approaches such as low-rank models \cite{chen2021bayesian} and deep learning methods \cite{miao2021generative, wu2021inductive}. 
Typically, traffic flow data always show strong global and local trends with long- and short-term patterns \cite{li2015trend}. The global trends usually refer to as certain periodic and cyclical patterns, which can be well characterized by low-rank models. Unfortunately, conventional low-rank models fail to characterize the time series dynamics because the reconstruction of low-rank models (e.g., low-rank matrix completion (LRMC) \cite{candes2010matrix, cai2010singular}) is invariant to the permutation of rows and columns. Thus, recent studies presented low-rank time series completion models based on certain algebraic structures, including Hankel matrices/tensors \cite{yokota2018missing, cai2019fast, cai2021accelerated}, circulant matrices/tensors \cite{yamamoto2022fast, liu2022recovery}, and convolution matrices \cite{liu2022recovery, liu2022time}. The sequential dependencies are implicitly captured by these structures when characterizing the low-rank property of time series. However, low-rank Hankel/convolution models are always limited to small- or middle-scale problems due to the large size of algebraic structures. Although circulant matrix nuclear norm minimization (CircNNM) can be efficiently solved through the fast Fourier transform (FFT), circulant matrices are restricted and fail to characterize the local trends of time series \cite{liu2022recovery}.



The default structure of most low-rank models---rank minimization whose objective is essentially converted into the nuclear norm minimization and takes a singular value thresholding \cite{candes2010matrix, cai2010singular}---does not ensure local smoothness. Thus, it requires us to model both global and local trends in a unified framework. In the literature, there are several ways to characterize the local spatial and temporal dependencies in data-driven machine learning models. For example, on the spatial dimension, Laplacian regularization has become a standard technique to impose local consistency (see e.g., \cite{cai2010graph, mao2018spatio}). On the temporal dimension, the local smoothness is often characterized using time series smoothing and autoregression explicitly (see e.g., \cite{xiong2010temporal, yu2016temporal, chen2022nonstationary}). Considering the importance of temporal regularization in regulating the behavior of global low-rank models, we are inspired to develop a tailored regularization for CircNNM to reinforce local trends while maintaining the algorithm's efficiency through FFT. Thus, using the fact that the Laplacian matrix of a circulant graph is a circulant matrix, we first introduce a novel temporal regularization to CircNNM. Next, we develop an efficient algorithm to solve the proposed Laplacian convolutional representation (LCR) using the alternating direction method of multipliers (ADMM). 
The contribution of this work is three-fold:
\begin{itemize}
\item We introduce a circular Laplacian kernel and use it to define a temporal regularization for characterizing the local trends in time series. By doing so, the temporal regularization can be formulated with circular convolution and draw the connection with FFT.
\item We propose a low-rank completion model---LCR---by characterizing global trends of sparse traffic time series as the nuclear norm of a circulant matrix and modeling local trends by the temporal regularization simultaneously. According to the properties of the circulant matrix and circular convolution, we present a fast implementation of LCR through FFT.
\item We empirically verify the importance of temporal regularization in LCR. 
Further, due to the fast implementation via FFT with log-linear time complexity, LCR is scalable to large imputation problems. The experimental results demonstrate that LCR performs better than the state-of-the-art baseline methods in terms of both accuracy and efficiency.
\end{itemize}

In practice, there is often a significant gap between accurate estimation and efficient implementation in the substantial imputation methods due to methodological challenges. For instance, while introducing algebraic structures \cite{yokota2018missing, yamamoto2022fast, liu2022time, liu2022recovery} or spatiotemporal smoothness \cite{chen2019abayesian,yu2016temporal} usually enhances accuracy, these approaches often come with a significantly higher computational cost. To address this gap, this work advances low-rank completion methods in the following ways. 
(i) The basic modeling idea of LCR stems from CircNNM \cite{liu2022recovery, liu2022time}, but LCR reinforces local time series trends with temporal regularization. (ii) LCR advances convolution nuclear norm minimization (ConvNNM, capable of global/local trend modeling depending on the kernel size) \cite{liu2022recovery, liu2022time} with a flexible time series modeling mechanism. (iii) LCR follows an efficient FFT implementation as CircNNM, while FFT cannot be used in ConvNNM (see e.g., Fig.~\ref{empirical_time_complexity_curve_convnnm_vs_lcr} for the computational cost). (iv) The flipping operation (see Fig.~\ref{flip_matrix}) in LCR addresses the issue caused by the correlation between the start and end data points of the time series. 

The remainder of this paper is structured as follows. Section~\ref{sec:related_work} and \ref{preliminaries} introduce the related work and some basic concepts, respectively. 
In Section~\ref{methodology}, we integrate the temporal regularization into low-rank models for characterizing both global and local trends in traffic time series. Section~\ref{univariate_experiment} and Section~\ref{multivariate_experiment} conduct imputation experiments on several real-world datasets. Finally, we conclude this study in Section~\ref{conclusion}.

\section{Related Work}\label{sec:related_work}


\subsection{Low-Rank Completion with Algebraic Structures}

Recent studies show great interest in time series completion with certain algebraic structures, e.g., the Hankel and circulant matrices. These approaches overcome some critical limitations of pure LRMC models, e.g., (i) LRMC is incapable of handling the entire row/column missing, (ii) LRMC is invariant to the permutation of rows/columns, and (iii) LRMC is not applicable to the case of univariate time series. For example, the model developed by Yokota \emph{et al}. \cite{yokota2018missing} can recover the missing slices of tensor using Hankel structure. Sedighin \emph{et al}. \cite{sedighin2020matrix} applied the tensor train decomposition to tensors obtained by multi-way Hankel structures and found better completion performance. In both studies, the low-rank methods on Hankel structures---replicating the data with certain sliding rules---are well-suited to modeling spatiotemporal dependencies and learning from sparse data.

A critical property of a circulant matrix is that its nuclear norm can be efficiently obtained via FFT. Using this property, Yamamoto \emph{et al}. \cite{yamamoto2022fast} proposed a fast tensor completion method; Liu and Zhang \cite{liu2022recovery} used the nuclear norm minimization of circulant matrices for missing data recovery and time series forecasting. Despite the fast algorithm for the circulant matrix, circulant-matrix-based models are inadequate in capturing the local trend/continuity in time series. Therefore, the ConvNNM model shows better local trends modeling if the kernel size of the convolution matrix is set as a relatively small value \cite{liu2022recovery}. Further, Liu \cite{liu2022time} proposed a learnable and orthonormal transformation for ConvNNM to reinforce its modeling ability when the convolutional low-rankness condition is not fully satisfied.



\subsection{Imputation with Temporal Modeling}

In the literature, considerable research has leveraged temporal dynamics in low-rank models for time series imputation. A common assumption among these models is that time series and their low-rank factors have local dependencies. 
Chen and Cichocki \cite{chen2005nonnegative} proposed a Toeplitz-matrix-based regularization to impose temporal smoothness in matrix factorization; the regularizer handles the difference between the low-rank factors of consecutive times. A similar regularizer based on the Toeplitz matrix was used by \cite{wang2018traffic} for traffic data reconstruction. Chen \emph{et al.} \cite{chen2021scalable} applied a quadratic variation (QV) to a traffic tensor completion problem to ensure temporal smoothness. For modeling low-dimensional temporal dynamics, Xiong \emph{et al.} \cite{xiong2010temporal} formulated a Bayesian tensor factorization with first-order Markovian assumptions on the temporal factors. Yu \emph{et al.} \cite{yu2016temporal} developed a temporal autoregressive regularizer in matrix factorization. Chen \emph{et al.} \cite{chen2021low} developed a low-rank autoregressive tensor completion model for traffic data imputation. While these two works assume the univariate autoregression, Chen and Sun \cite{chen2021bayesian} applied a vector autoregression on the latent temporal factors and developed a fully Bayesian model for multidimensional and sparse time series prediction. 
As Laplacian regularization is of broad use in graph modeling, it is also applicable to temporal modeling. For example, Rao \emph{et al.} \cite{rao2015collaborative} proposed a matrix completion algorithm with Laplacian regularization. To the best of our knowledge, we are the first to present a Laplacian kernelized temporal regularization with circular convolution which consequently leads to the use of FFT. 


\section{Preliminaries}\label{preliminaries}

In this section, we introduce the basic definitions of the circulant matrix, convolution matrix, and circular convolution, in the meanwhile summarizing their relationships. 

\subsection{Circulant Matrix}

The circulant matrix is an important structure that shows broad use in the field of signal processing \cite{hansen2006deblurring, wright2022high}. By definition, for any vector $\boldsymbol{x}=(x_1,x_2,\cdots,x_{T})^\top\in\mathbb{R}^{T}$, the circulant matrix can be written as follows,
\begin{equation}\label{circulant_matrix}
\mathcal{C}(\boldsymbol{x})\triangleq\begin{bmatrix}
x_{1} & x_{T} & x_{T-1} & \cdots & x_{2} \\
x_{2} & x_{1} & x_{T} & \cdots & x_{3} \\
x_{3} & x_{2} & x_{1} & \cdots & x_{4} \\
\vdots & \vdots & \vdots & \ddots & \vdots \\
x_{T} & x_{T-1} & x_{T-2} & \cdots & x_{1} \\
\end{bmatrix}\in\mathbb{R}^{T\times T},
\end{equation}
where $\mathcal{C}:\mathbb{R}^{T}\to\mathbb{R}^{T\times T}$ denotes the circulant operator. The first column of $\mathcal{C}(\boldsymbol{x})$ is the vector $\boldsymbol{x}$, and the diagonal entries of $\mathcal{C}(\boldsymbol{x})$ are all equal to $x_1$. 

\subsection{Convolution Matrix}

Convolution is vital to a variety of machine learning problems. By definition, for any vectors $\boldsymbol{x}=(x_1,x_2,\cdots,x_{T})^\top\in\mathbb{R}^{T}$ and $\boldsymbol{y}=(y_1,y_2,\cdots,y_{\tilde{\tau}})^\top\in\mathbb{R}^{\tilde{\tau}}$ with $\tilde{\tau}\leq T$, the circular convolution of two vectors is $\boldsymbol{z}=\boldsymbol{x}\star\boldsymbol{y}\in\mathbb{R}^{T}$ \cite{brunton2022data}, denoting the operator with the symbol $\star$; element-wise, we have
\begin{equation}
z_{t}=\sum_{k=1}^{\tilde{\tau}}x_{t-k+1} y_{k},\,\forall t\in\{1,2,\ldots,T\},
\end{equation}
where $z_{t}$ is the $t$th entry of $\boldsymbol{z}$ and $x_{t-k+1}=x_{t-k+1+T}$ for $t+1\leq k$. In particular, circular convolution is a linear operator that can be expressed $\boldsymbol{x}\star\boldsymbol{y}\equiv\mathcal{C}_{\tilde{\tau}}(\boldsymbol{x})\boldsymbol{y}$ where $\mathcal{C}_{\tilde{\tau}}:\mathbb{R}^{T}\to\mathbb{R}^{T\times \tilde{\tau}}$ denotes the convolution operator with kernel size $\tilde{\tau}$. The resultant convolution matrix consists of the first $\tilde{\tau}$ columns of the circulant matrix $\mathcal{C}(\boldsymbol{x})$ \cite{liu2022recovery, liu2022time}. Given any vectors $\boldsymbol{x},\boldsymbol{y}\in\mathbb{R}^{T}$ of the same length, then we have $\boldsymbol{x}\star\boldsymbol{y}\equiv\mathcal{C}(\boldsymbol{x})\boldsymbol{y}$.

\section{Methodology}\label{methodology}

In this section, we introduce an efficient LCR model for imputing sparse traffic time series. 
To resolve the optimization problem, we seek an FFT implementation in the frequency domain within the ADMM, instead of the time domain.

\subsection{Laplacian Kernel}

Laplacian matrix is a classical structure for representing the links among nodes in a graph. In this work, we extract the temporal dependencies of time series through undirected and circulant graphs. Recall that the Laplacian matrix by definition takes $\boldsymbol{L}=\boldsymbol{D}-\boldsymbol{A}$ in which $\boldsymbol{D}$ and $\boldsymbol{A}$ are the (diagonal) degree matrix and adjacency matrix, respectively \cite{sandryhaila2013discrete}. 
In Fig.~\ref{laplacian_examples}, the Laplacian matrices of both graphs are circulant matrices, and their first columns are $\boldsymbol{\ell}=(2,-1,0,0,-1)^\top$ and $\boldsymbol{\ell}=(4,-1,-1,-1,-1)^\top$, respectively.

\begin{figure}[!ht]
\centering
\subfigure[Circulant graph with degree 2.]{
    \centering
    \includegraphics[width = 0.225\textwidth]{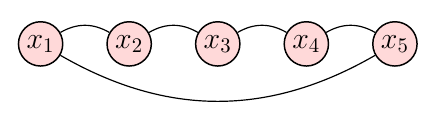}\label{laplacian_examples_degree_2}
}
\subfigure[circulant graph with degree 4.]{
    \centering
    \includegraphics[width = 0.225\textwidth]{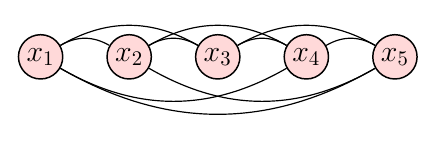}
}
\caption{Undirected and circulant graphs on the relational data samples $\{{x}_{1},{x}_{2},\ldots,{x}_{5}\}$ with certain degrees.}
\label{laplacian_examples}
\end{figure}

In this work, we introduce a Laplacian kernel as described in Definition~\ref{laplacian_kernel_def}, allowing one to characterize temporal dependencies of time series. In aforementioned cases as shown in Fig.~\ref{laplacian_examples}, the first column of the Laplacian matrix $\boldsymbol{L}$ is indeed a simple example of Laplacian kernel.

\begin{definition}[Laplacian Kernel]\label{laplacian_kernel_def}
Given any time series $\boldsymbol{x}\in\mathbb{R}^{T}$, suppose $\tau\in\mathbb{Z}^{+}$ ($\tau\leq\frac{1}{2}(T-1)$) be the kernel size of an undirected and circulant graph, then the Laplacian kernel is defined as
\begin{equation}
\boldsymbol{\ell}\triangleq(2\tau,\underbrace{-1,\cdots,-1}_{\tau},0,\cdots,0,\underbrace{-1,\cdots,-1}_{\tau})^\top\in\mathbb{R}^{T},
\end{equation}
which is also the first column of the Laplacian matrix and the inherent degree matrix is diagonalized with entries $2\tau$.
\end{definition}

\begin{remark}\label{flip_vec_remark}
The circulant operation assumes that the start data points and the end data points are connected, which is a disadvantage in real-world data analysis. To overcome this issue, on any time series $\boldsymbol{x}\in\mathbb{R}^{T}$, one can construct the following vector:
\begin{equation}
\boldsymbol{x}_{\operatorname{new}}=\begin{bmatrix}
\boldsymbol{x} \\ \boldsymbol{J}_{T}\boldsymbol{x} \\
\end{bmatrix}=(x_1,\cdots,x_T,x_T,\cdots,x_1)^\top\in\mathbb{R}^{2T},
\end{equation}
where $\boldsymbol{J}_{T}\in\mathbb{R}^{T\times T}$ is the exchange matrix whose antidiagonal entries are one and other entries are zero. 
\end{remark}



Typically, the temporal regularization calculates how values of $\boldsymbol{x}$ differ from their adjacent values, and can thus be used as a regularization of local temporal smoothness. According to the definitions of Laplacian kernel and circulant matrix and the relationship between circular convolution and circulant matrix, we declare the following form:
\begin{equation}\label{laplacian_regularization}
\mathcal{R}_\tau(\boldsymbol{x})=\frac{1}{2}\|\boldsymbol{L}\boldsymbol{x}\|_{2}^2=\frac{1}{2}\|\mathcal{C}(\boldsymbol{\ell}) \boldsymbol{x} \|_2^2=\frac{1}{2}\|\boldsymbol{\ell}\star\boldsymbol{x}\|_{2}^2.
\end{equation}
where $\|\cdot\|_2$ denotes the $\ell_2$-norm of a vector. As can be seen, the Laplacian kernel can represent the graphical relationship in the Laplacian matrix, showing no need for constructing the Laplacian matrix anymore. The setting of the degree $2\tau$ depends on the strength of local dependencies and the missing data scenarios (e.g., missing rate).

\begin{remark}
Definition~\ref{laplacian_kernel_def} allows one to obtain a more flexible design of the kernel $\boldsymbol{\ell}$. For example, if we introduce a directed Laplacian kernel in the form of random walk \cite{xiong2010temporal, cai2010graph} (i.e., $\boldsymbol{\ell}=(1,0,\cdots,0,-1)^\top\in\mathbb{R}^{T}$), then the temporal regularization is equivalent to the QV regularization, namely,
\begin{equation}\label{QV_laplacian_matrix}
\frac{1}{2}\|\boldsymbol{\ell}\star\boldsymbol{x}\|_{2}^{2}=\frac{1}{2}\boldsymbol{x}^\top\tilde{\boldsymbol{L}}\boldsymbol{x},
\end{equation}
where $\tilde{\boldsymbol{L}}\in\mathbb{R}^{T\times T}$ is a circulant matrix with the degree 2.
\end{remark}


\begin{theorem}[Convolution Theorem \cite{brunton2022data}]\label{convolution_theorem}
For any vectors $\boldsymbol{x},\boldsymbol{y}\in\mathbb{R}^{T}$, a circular convolution in the time domain is a product in the frequency domain, and it always holds that
\begin{equation}
\boldsymbol{x}\star\boldsymbol{y}=\mathcal{F}^{-1}(\mathcal{F}(\boldsymbol{x})\circ\mathcal{F}(\boldsymbol{y})),
\end{equation}
where $\mathcal{F}(\cdot)$ and $\mathcal{F}^{-1}(\cdot)$ denote the discrete Fourier transform (DFT) and the inverse DFT, respectively. $\mathcal{F}(\boldsymbol{x}),\mathcal{F}(\boldsymbol{y})\in\mathbb{C}^{T}$ are the results of DFT on $\boldsymbol{x},\boldsymbol{y}$ with $\mathbb{C}$ denoting the set of complex numbers. The symbol $\circ$ denotes the Hadamard product.
\end{theorem}

Essentially, Theorem~\ref{convolution_theorem} describes the relationship between circular convolution and DFT, showing that the circular convolution can be implemented in the frequency domain. The circulant matrix is advantageous because the required matrix-vector product can usually be done efficiently by leveraging the structure. Since the Laplacian kernel stems from circulant matrices, the temporal regularization in Eq.~\eqref{laplacian_regularization} can be thus reformulated as follows,
\begin{equation}\label{temporal_regularization_Fourier}
\mathcal{R}_\tau(\boldsymbol{x})=\frac{1}{2}\|\boldsymbol{\ell}\star\boldsymbol{x}\|_{2}^2=\frac{1}{2T}\|\mathcal{F}(\boldsymbol{\ell})\circ\mathcal{F}(\boldsymbol{x})\|_{2}^{2}.
\end{equation}


\begin{remark}
We can prove Eq.~\eqref{temporal_regularization_Fourier} as follows. Let
\begin{equation}
\begin{cases}
\boldsymbol{\alpha}=\boldsymbol{\ell}\star\boldsymbol{x},\\
\boldsymbol{\beta}=\mathcal{F}(\boldsymbol{\ell})\circ\mathcal{F}(\boldsymbol{x}),
\end{cases}
\end{equation}
then it takes $\boldsymbol{\alpha}=\mathcal{F}^{-1}(\boldsymbol{\beta})$ and $\mathcal{F}(\boldsymbol{\alpha})=\boldsymbol{\beta}$ (see Theorem~\ref{convolution_theorem}). Thus, according to the Parseval's theorem \cite{brunton2022data}, we get
\begin{equation}
\|\boldsymbol{\alpha}\|_{2}^{2}=\frac{1}{T}\|\mathcal{F}(\boldsymbol{\alpha})\|_{2}^{2}=\frac{1}{T}\|\boldsymbol{\beta}\|_{2}^{2},
\end{equation}
as claimed in Eq.~\eqref{temporal_regularization_Fourier}.
\end{remark}

\subsection{Univariate Time Series Imputation}

\subsubsection{Problem Definition}

Spatiotemporal traffic data modeling is vital to several ITS applications. Typically, traffic flow data by nature involve certain time series characteristics, e.g., global daily/weekly rhythm and local trends. However, such kind of time series are usually incomplete or even sparse due to unpredictable data collection processes. In the univariate case, the imputation problem can be summarized as Problem~\ref{univariate_imputation_problem}.

\begin{problem}[Univariate Time Series Imputation]\label{univariate_imputation_problem}
For any partially observed time series $\boldsymbol{y}\in\mathbb{R}^{T}$ with observed index set $\Omega$, the goal is to impute the missing data $\mathcal{P}_{\Omega}^\perp(\boldsymbol{y})$ from $\mathcal{P}_{\Omega}(\boldsymbol{y})$. Herein, $\mathcal{P}_{\Omega}:\mathbb{R}^{T}\to\mathbb{R}^{T}$ denotes the orthogonal projection supported on $\Omega$, while $\mathcal{P}_{\Omega}^\perp:\mathbb{R}^{T}\to\mathbb{R}^{T}$ denotes the orthogonal projection supported on the complement of $\Omega$.
\end{problem}

\begin{remark}
On the vector $\boldsymbol{y}\in\mathbb{R}^{T}$ with observed index set $\Omega$, the operator $\mathcal{P}_{\Omega}(\cdot)$ can be described as follows,
\begin{equation}
[\mathcal{P}_{\Omega}(\boldsymbol{y})]_{t}=\begin{cases}
y_{t},&\text{if $t\in\Omega$,} \\
0,&\text{otherwise,} \\
\end{cases}
\end{equation}
where $t=1,2,\ldots,T$.
\end{remark}

\subsubsection{Model Description}

Although ConvNNM and CircNNM can reconstruct missing values in time series, both models fail to incorporate global and local consistency appropriately. 
In this work, we propose the LCR imputation model, in which we utilize circulant matrix nuclear norm to pursue the global trends and use the temporal regularization to characterize the local trends in time series (see Fig.~\ref{framework} for an illustration). Formally, the LCR model can be formulated as follows,
\begin{equation}\label{lcr_opt_prob_v1}
\begin{aligned}
\min_{\boldsymbol{x}}~&\|\mathcal{C}(\boldsymbol{x})\|_{*}+\gamma\cdot\mathcal{R}_\tau(\boldsymbol{x}) \\
\text{s.t.}~~&\mathcal{P}_{\Omega}(\boldsymbol{x})=\mathcal{P}_{\Omega}(\boldsymbol{y}),
\end{aligned}
\end{equation}
where $\|\cdot\|_{*}$ denotes the nuclear norm of matrix (i.e., the sum of singular values). 
The vector $\boldsymbol{x}\in\mathbb{R}^T$ is the reconstructed time series corresponding to the partially observed time series $\boldsymbol{y}$. 
In the objective function, $\gamma$ is the weight parameter.

\begin{figure}[!ht]
    \centering
    \includegraphics[width=0.45\textwidth]{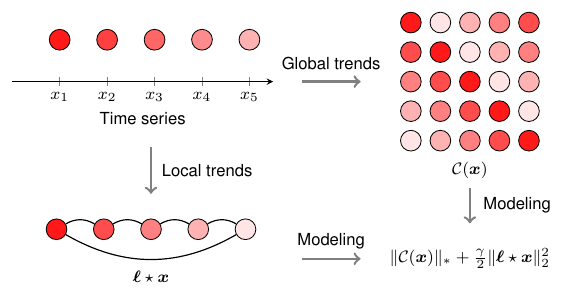}
    \caption{Illustration of the proposed LCR model.}
    \label{framework}
\end{figure}

Since traffic time series data are usually noisy, the strong observation constraint in Eq.~\eqref{lcr_opt_prob_v1} should be replaced by $\|\mathcal{P}_{\Omega}(\boldsymbol{z}-\boldsymbol{y})\|_{2}\leq \epsilon$ in which $\epsilon\geq0$ is the tolerance. Now, the optimization problem of LCR is given by
\begin{equation}\label{lcr_opt_prob}
\begin{aligned}
\min_{\boldsymbol{x}}~&\|\mathcal{C}(\boldsymbol{x})\|_{*}+\gamma\cdot\mathcal{R}_\tau(\boldsymbol{x}) \\
\text{s.t.}~~&\|\mathcal{P}_{\Omega}(\boldsymbol{x}-\boldsymbol{y})\|_{2}\leq \epsilon.
\end{aligned}
\end{equation}

Our LCR model stems from ConvNNM \cite{liu2022recovery, liu2022time}, and it can be solved by the ADMM framework. To resolve the convex optimization problem of LCR in Eq.~\eqref{lcr_opt_prob}, we introduce an auxiliary variable $\boldsymbol{z}$ to preserve the observation information. Thus, the optimization problem becomes
\begin{equation}\label{lcr_opt_prob_v2}
\begin{aligned}
\min_{\boldsymbol{x},\boldsymbol{z}}~&\|\mathcal{C}(\boldsymbol{x})\|_{*}+\gamma\cdot\mathcal{R}_\tau(\boldsymbol{x})+\eta\cdot\pi(\boldsymbol{z}) \\
\text{s.t.}~~&\boldsymbol{x}=\boldsymbol{z},
\end{aligned}
\end{equation}
where $\eta$ is a weight parameter. We define $\pi(\cdot)$ as the reconstructed errors between $\boldsymbol{z}$ and $\boldsymbol{y}$ in the set $\Omega$, which is formally given by $\pi(\boldsymbol{z})=\frac{1}{2}\|\mathcal{P}_{\Omega}(\boldsymbol{z}-\boldsymbol{y})\|_{2}^{2}$. To reinforce both global and local trends in the reconstructed time series $\boldsymbol{x}$, the observation constraint can be related to the noisy version as shown in Eq.~\eqref{lcr_opt_prob}, thus leading to the denoised and smooth time series in $\boldsymbol{x}$. Accordingly, the augmented Lagrangian function of Eq.~\eqref{lcr_opt_prob_v2} can be written as follows,
\begin{equation}\label{augmented_lagrangian_func}
\begin{aligned}
\mathcal{L}(\boldsymbol{x},\boldsymbol{z},\boldsymbol{w})=&\|\mathcal{C}(\boldsymbol{x})\|_{*}+\gamma\cdot\mathcal{R}_\tau(\boldsymbol{x})+\frac{\lambda}{2}\|\boldsymbol{x}-\boldsymbol{z}\|_{2}^{2}\\
&+\langle\boldsymbol{w},\boldsymbol{x}-\boldsymbol{z}\rangle+\eta\cdot\pi(\boldsymbol{z}),
\end{aligned}
\end{equation}
where $\boldsymbol{w}\in\mathbb{R}^{T}$ is the Lagrange multiplier, and $\lambda$ is a hyperparameter. The symbol $\langle\cdot,\cdot\rangle$ denotes the inner product. Note that the constraint $\boldsymbol{x}=\boldsymbol{z}$ in the optimization problem is relaxed by the Lagrange multiplier.

Thus, the ADMM scheme can be summarized as follows,
\begin{equation}\label{admm_scheme}
\begin{cases}
\displaystyle
\boldsymbol{x}:=\argmin_{\boldsymbol{x}}\mathcal{L}(\boldsymbol{x},\boldsymbol{z},\boldsymbol{w}), \\
\displaystyle
\boldsymbol{z}:=\argmin_{\boldsymbol{z}}\mathcal{L}(\boldsymbol{x},\boldsymbol{z},\boldsymbol{w}), \\
\boldsymbol{w}:=\boldsymbol{w}+\lambda(\boldsymbol{x}-\boldsymbol{z}), \\
\end{cases}
\end{equation}
which is a two-block ADMM. Since the objective function of Eq.~\eqref{lcr_opt_prob_v2} is the sum of two separable convex functions, the convergence of LCR can be proved as in \cite{chen2016direct}.



\subsubsection{Estimating the Variable $\boldsymbol{x}$}

In particular, with respect to the variable $\boldsymbol{x}$, we can rewrite the regularization terms in Eq.~\eqref{augmented_lagrangian_func} as follows,
\begin{equation}\label{remaining_regularization_terms}
\begin{aligned}
f=&\frac{\gamma}{2}\|\boldsymbol{\ell}\star\boldsymbol{x}\|_{2}^2+\frac{\lambda}{2}\|\boldsymbol{x}-\boldsymbol{z}+\boldsymbol{w}/\lambda\|_{2}^{2} \\
=&\frac{\gamma}{2T}\|\hat{\boldsymbol{\ell}}\circ\hat{\boldsymbol{x}}\|_{2}^2+\frac{\lambda}{2T}\|\hat{\boldsymbol{x}}-\hat{\boldsymbol{z}}+\hat{\boldsymbol{w}}/\lambda\|_{2}^{2}, \\
\end{aligned}
\end{equation}
where $\hat{\boldsymbol{\ell}}=\mathcal{F}(\boldsymbol{\ell})$, and we introduce the variables $\{\hat{\boldsymbol{x}},\hat{\boldsymbol{z}},\hat{\boldsymbol{w}}\}=\{\mathcal{F}(\boldsymbol{x}),\mathcal{F}(\boldsymbol{z}),\mathcal{F}(\boldsymbol{w})\}$ referring to $\{\boldsymbol{x},\boldsymbol{z},\boldsymbol{w}\}$ in the frequency domain. 
Notably, temporal regularization $\mathcal{R}_{\tau}(\boldsymbol{x})$ can be converted into a DFT copy (see Eq.~\eqref{temporal_regularization_Fourier}), and the Parseval's theorem is also applicable to the remaining term of $f$. 

\begin{lemma}\label{convolution_nuclear_norm}
For any vector $\boldsymbol{x}\in\mathbb{R}^{T}$, the nuclear norm of the resultant circulant matrix $\mathcal{C}(\boldsymbol{x})\in\mathbb{R}^{T\times T}$ is related to the DFT:
\begin{equation}
\|\mathcal{C}(\boldsymbol{x})\|_{*}=\|\mathcal{F}(\boldsymbol{x})\|_{1}.
\end{equation}
\end{lemma}

\begin{proof}
For any circulant matrix $\mathcal{C}(\boldsymbol{x})$ if and only if it is diagonalizable by the unitary matrix, the eigenvalue decomposition \cite{wright2022high} can be written as follows,
\begin{equation}
\mathcal{C}(\boldsymbol{x})=\boldsymbol{U}\text{diag}(\mathcal{F}(\boldsymbol{x}))\boldsymbol{U}^{H},
\end{equation}
where $\cdot^H$ denotes the conjugate transpose. Since $\boldsymbol{U}$ is a unitary matrix, it always holds that
\begin{equation*}
\begin{aligned}
\|\mathcal{C}(\boldsymbol{x})\|_{*}=&\|\boldsymbol{U}\text{diag}(\mathcal{F}(\boldsymbol{x}))\boldsymbol{U}^{H}\|_{*}
=\|\text{diag}(\mathcal{F}(\boldsymbol{x}))\|_{*}
=\|\mathcal{F}(\boldsymbol{x})\|_{1},
\end{aligned}
\end{equation*}
and we can calculate the singular values of $\mathcal{C}(\boldsymbol{x})$ from the FFT of $\boldsymbol{x}$. Here, FFT is an efficient algorithm for computing the DFT in $\mathcal{O}(T\log T)$ time.
\end{proof}

Going back to the ADMM scheme in Eq.~\eqref{admm_scheme} and using the property of circulant matrix nuclear norm in Lemma~\ref{convolution_nuclear_norm}, the $\boldsymbol{x}$-subproblem 
can be converted into the optimization over the variable $\hat{\boldsymbol{x}}$ in the frequency domain. Thus,
\begin{equation}
\begin{aligned}
\boldsymbol{x}:=\argmin_{\boldsymbol{x}}~&\|\mathcal{C}(\boldsymbol{x})\|_{*}+\frac{\gamma}{2}\|\boldsymbol{\ell}\star\boldsymbol{x}\|_{2}^2+\frac{\lambda}{2}\|\boldsymbol{x}-\boldsymbol{z}+\boldsymbol{w}/\lambda\|_{2}^{2},
\end{aligned}
\end{equation}
is equivalent to
\begin{equation}\label{opt_prob_complex_space_0}
\hat{\boldsymbol{x}}:=\argmin_{\hat{\boldsymbol{x}}}~\|\hat{\boldsymbol{x}}\|_1+\frac{\gamma}{2T}\|\hat{\boldsymbol{\ell}}\circ\hat{\boldsymbol{x}}\|_2^2+\frac{\lambda}{2T}\|\hat{\boldsymbol{x}}-\hat{\boldsymbol{z}}+\hat{\boldsymbol{w}}/\lambda\|_2^2.
\end{equation}

On each $\hat{x}_t$, the optimization problem is given by
\begin{equation}\label{opt_prob_complex_space}
\begin{aligned}
\hat{x}_t:=\argmin_{\hat{x}_t}~&|\hat{x}_t|+\frac{\gamma}{2T}|\hat{\ell}_t\hat{x}_t|^2+\frac{\lambda}{2T}|\hat{x}_t-\hat{z}_t+\hat{w}_t/\lambda|^2 \\
=\argmin_{\hat{x}_t}~&|\hat{x}_t|+\frac{\gamma|\hat{\ell}_t|^2+\lambda}{2T}\Bigl|\hat{x}_t-\frac{\lambda\hat{z}_t-\hat{w}_t}{\gamma|\hat{\ell}_t|^2+\lambda}\Bigr|^2, \\
\end{aligned}
\end{equation}
where $|\hat{\ell}_t\hat{x}_t|^2=|\hat{\ell}_t|^2\cdot|\hat{x}_t|^2$.

The resultant $\ell_1$-norm minimization is memory-efficient, easy to compute, and preserves the singular values of circulant matrix that are due to the FFT. The closely related analysis and results are also discussed in \cite{yamamoto2022fast, liu2022recovery, liu2022time}.

According to Eqs.~\eqref{opt_prob_complex_space_0} and \eqref{opt_prob_complex_space}, we let
\begin{equation}\label{update_h}
\hat{\boldsymbol{h}}\triangleq(\lambda\hat{\boldsymbol{z}}-\hat{\boldsymbol{w}})\oslash(\gamma\hat{\boldsymbol{\ell}}^{*}\circ\hat{\boldsymbol{\ell}}+\lambda\mathbbm{1}_{T}),
\end{equation}
where $\cdot^{*}$ represents the complex conjugate, and $\oslash$ denotes the Hadamard division. $\mathbbm{1}_{T}\in\mathbb{R}^{T}$ is the vector of ones.

The closed-form solution to $\hat{\boldsymbol{x}}$ can be found in Lemma~\ref{L1_norm_complex_space}. As we have the closed-form solution as described in Eq.~\eqref{prox_L1_norm_complex} (i.e., with respect to each $\hat{x}_t$) such that
\begin{equation}\label{L1_norm_thresholding}
\hat{x}_t:=\frac{\hat{h}_{t}}{|\hat{h}_t|}\cdot\max\{0,|\hat{h}_t|-1/\delta_t\},
\end{equation}
with
\begin{equation}
\left\{
\begin{aligned}
\delta_t&\triangleq(\gamma|\hat{\ell}_t|^2+\lambda)/T, \\
\hat{h}_t&\triangleq(\lambda\hat{z}_t-\hat{w}_t)/(\gamma|\hat{\ell}_t|^2+\lambda), \\
\end{aligned}
\right.
\end{equation}
we can therefore update the variable $\boldsymbol{x}$ by
\begin{equation}\label{update_x}
\boldsymbol{x}:=\mathcal{F}^{-1}(\hat{\boldsymbol{x}}).
\end{equation}

\begin{lemma}\label{L1_norm_complex_space}
Following Eqs.~\eqref{opt_prob_complex_space_0} and \eqref{opt_prob_complex_space}, for any $\ell_1$-norm minimization problem in complex space such that
\begin{equation}
\min_{\hat{\boldsymbol{x}}}~\|\hat{\boldsymbol{x}}\|_{1}+\frac{\delta}{2}\|\hat{\boldsymbol{x}}-\hat{\boldsymbol{h}}\|_{2}^{2},
\end{equation}
with complex-valued vectors $\hat{\boldsymbol{x}},\hat{\boldsymbol{h}}\in\mathbb{C}^T$ and weight parameter $\delta\in\mathbb{R}$, element-wise, the solution is given by
\begin{equation}\label{prox_L1_norm_complex}
\hat{x}_t:=\frac{\hat{h}_{t}}{|\hat{h}_t|}\cdot\max\{0,|\hat{h}_t|-1/\delta\},t=1,\ldots,T.
\end{equation}
\end{lemma}

\begin{proof}
In theory, Lemma~\ref{L1_norm_complex_space} invokes the shrinkage operator in \cite{yang2009fast, liu2012robust, liu2022recovery}.
\end{proof}

\subsubsection{Estimating the Variable $\boldsymbol{z}$}

In the ADMM scheme (see Eq.~\eqref{admm_scheme}), the subproblem with respect to the variable $\boldsymbol{z}$ can be written as follows,
\begin{equation}
\min_{\boldsymbol{z}}~\frac{\lambda}{2}\|\boldsymbol{x}-\boldsymbol{z}-\boldsymbol{w}/\lambda\|_{2}^{2}+\frac{\eta}{2}\|\mathcal{P}_{\Omega}(\boldsymbol{z}-\boldsymbol{y})\|_{2}^{2}.
\end{equation}
Let $g$ be the objective function, then the partial derivative with respect to $\boldsymbol{z}$ can be formed by $\mathcal{P}_{\Omega}(\boldsymbol{z})$ and $\mathcal{P}_{\Omega}^\perp(\boldsymbol{z})$:
\begin{equation}
\left\{
\begin{aligned}
{\partial g}/{\partial\mathcal{P}_{\Omega}(\boldsymbol{z})}=&\lambda\mathcal{P}_{\Omega}(\boldsymbol{z}-\boldsymbol{x}-\boldsymbol{w}/\lambda)+\eta\mathcal{P}_{\Omega}(\boldsymbol{z}-\boldsymbol{y}), \\
{\partial g}/{\partial\mathcal{P}_{\Omega}^\perp(\boldsymbol{z})}=&\lambda\mathcal{P}_{\Omega}^\perp(\boldsymbol{z}-\boldsymbol{x}-\boldsymbol{w}/\lambda).
\end{aligned}\right.
\end{equation}

As a result, ${\partial g}/{\partial \boldsymbol{z}}=\boldsymbol{0}$ produces a closed-form solution: 
\begin{equation}\label{update_z}
\begin{aligned}
\boldsymbol{z}:&=\Bigl\{\boldsymbol{z}\mid {\partial g}/{\partial\mathcal{P}_{\Omega}(\boldsymbol{z})}+{\partial g}/{\partial\mathcal{P}_{\Omega}^\perp(\boldsymbol{z})}=\boldsymbol{0}\Bigr\} \\
&=\frac{1}{\lambda+\eta}\mathcal{P}_{\Omega}(\lambda\boldsymbol{x}+\boldsymbol{w}+\eta\boldsymbol{y})+\frac{1}{\lambda}\mathcal{P}_{\Omega}^\perp(\lambda\boldsymbol{x}+\boldsymbol{w}).
\end{aligned}
\end{equation}
In this case, if $\eta\to+\infty$, then the solution refers to the LCR with strong observation constraint in Eq.~\eqref{lcr_opt_prob_v1}. In terms of the parameter $\eta$, we can preferably set its value to $\eta=c\cdot\lambda$ with $c\in\{10^2,10^3\}$ to preserve the observation information.

\subsubsection{Solution Algorithm}

As mentioned above, our LCR model reinforces the modeling processes of global low-rank and local temporal trends in time series data. Since we utilize the circulant matrix and circular convolution, it is not hard to show the appealing properties of DFT and lead to an elegant and fast solution algorithm. Algorithm~\ref{univariate_algorithm} summarizes the implementation of the proposed LCR model.

\begin{algorithm}
\caption{Laplacian Convolutional Representation (LCR)}
\label{univariate_algorithm}
\begin{algorithmic}[1]
\renewcommand{\algorithmicrequire}{\textbf{Input:}}
\renewcommand{\algorithmicensure}{\textbf{Output:}}
\REQUIRE Data $\boldsymbol{y}\in\mathbb{R}^{T}$ with observed index set $\Omega$, Laplacian kernel size $\tau\in\mathbb{Z}^{+}$, and hyperparameters $\{\gamma,\lambda,\eta\}$.
\ENSURE Reconstructed vector $\boldsymbol{x}\in\mathbb{R}^{T}$.
\STATE Initialize $\{\boldsymbol{x}_{0},\boldsymbol{z}_{0},\boldsymbol{w}_{0}\}$.
\STATE Construct the Laplacian kernel $\boldsymbol{\ell}$ with $\tau$ and perform FFT on it to get $\hat{\boldsymbol{\ell}}$.
\FOR {$i=0$ to maximum iteration}
\STATE Perform FFT on $\{\boldsymbol{z}_{i},\boldsymbol{w}_{i}\}$.
\STATE Compute $\hat{\boldsymbol{h}}$ by Eq.~\eqref{update_h}.
\STATE Compute $\hat{\boldsymbol{x}}$ by the shrinkage in Eq.~\eqref{L1_norm_thresholding}.
\STATE Compute $\boldsymbol{x}_{i+1}$ by $\boldsymbol{x}_{i+1}=\mathcal{F}^{-1}(\hat{\boldsymbol{x}})$ (see Eq.~\eqref{update_x}).
\STATE Compute $\boldsymbol{z}_{i+1}$ by Eq.~\eqref{update_z}.
\STATE Compute $\boldsymbol{w}_{i+1}=\boldsymbol{w}_{i}+\lambda(\boldsymbol{x}_{i+1}-\boldsymbol{z}_{i+1})$ (see Eq.~\eqref{admm_scheme}).
\ENDFOR
\end{algorithmic}
\end{algorithm}

To analyze the empirical time complexity of LCR (with 50 iterations by default), Fig.~\ref{empirical_time_complexity_curve} shows the running times of LCR on the generated data with different data lengths (i.e., data $\boldsymbol{y}\in\mathbb{R}^{T}$ with $T\in\{2^{10},2^{11},\ldots,2^{20}\}$). As shown in Fig.~\ref{empirical_time_complexity_curve_convnnm_vs_lcr}, we compare LCR with ConvNNM (e.g., kernel size $\tilde{\tau}=2^4$ in this case), demonstrating that LCR is more efficient than ConvNNM. Typically, ConvNNM can be converted into a standard nuclear norm minimization with singular value thresholding (i.e., of time complexity $\mathcal{O}(\tilde{\tau}^2T)$) \cite{cai2010singular, candes2010matrix, candes2012exact, liu2022recovery, liu2022time}. The computational cost of ConvNNM would increase with a larger $\tilde{\tau}$ for the convolution matrix $\mathcal{C}_{\tilde{\tau}}(\boldsymbol{y})\in\mathbb{R}^{T\times\tilde{\tau}}$. In contrast, both CircNNM and LCR have an efficient solution through FFT in $\mathcal{O}(T\log T)$ time. 

\begin{figure}[!ht]
\centering
\subfigure[ConvNNM vs. LCR]{
\centering
\includegraphics[width = 0.247\textwidth]{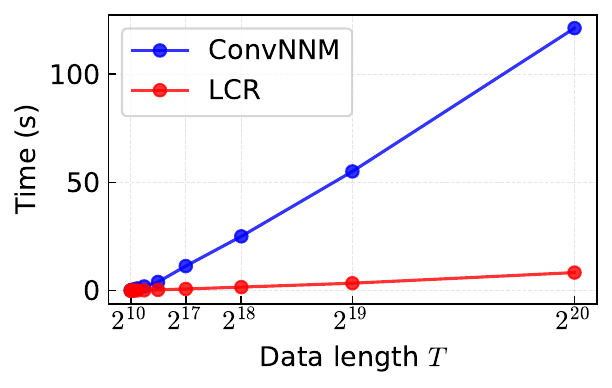}\label{empirical_time_complexity_curve_convnnm_vs_lcr}
}\hspace{-1.1em}
\subfigure[LCR]{
\centering
\includegraphics[width=0.232\textwidth]{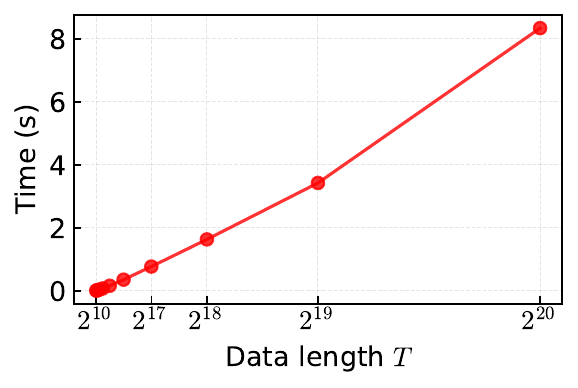}\label{empirical_time_complexity_curve_lcr}
}
\caption{Empirical time complexity. The model is tested 50 times on each generated data.}
\label{empirical_time_complexity_curve}
\end{figure}

\subsection{Multivariate Time Series Imputation}

\subsubsection{Problem Definition}

Considering both the spatial and temporal dimensions in traffic data, we have a multivariate time series imputation task as described in Problem~\ref{multivariate_imputation_problem}. The critical question is how to characterize both spatial and temporal dependencies of traffic time series data in the modeling process.

\begin{problem}[Multivariate Time Series Imputation]\label{multivariate_imputation_problem}
For any partially observed time series $\boldsymbol{Y}\in\mathbb{R}^{N\times T}$ with $N$ variables and $T$ time steps, if its observed index set is denoted by $\Omega$, then the goal is to impute the missing data $\mathcal{P}_{\Omega}^\perp(\boldsymbol{Y})$ from $\mathcal{P}_{\Omega}(\boldsymbol{Y})$. Herein, $\mathcal{P}_{\Omega}:\mathbb{R}^{N\times T}\to\mathbb{R}^{N\times T}$ and $\mathcal{P}_{\Omega}^\perp:\mathbb{R}^{N\times T}\to\mathbb{R}^{N\times T}$ denote the orthogonal projection supported on $\Omega$ and the complement of $\Omega$, respectively.
\end{problem}

\subsubsection{Model Description}

In this study, we introduce the nuclear norm of the circulant tensor $\mathcal{C}(\boldsymbol{X})$ in Definition~\ref{circulant_tensor_nuclear_norm_def}, which follows the tensor nuclear norm proposed in \cite{semerci2014tensor, lu2019tensor}. 

\begin{definition}[Circulant Tensor Nuclear Norm]\label{circulant_tensor_nuclear_norm_def}
For any matrix $\boldsymbol{X}\in\mathbb{R}^{N\times T}$, the corresponding circulant tensor is $\mathcal{C}(\boldsymbol{X})\in\mathbb{R}^{N\times N\times T\times T}$, which can be factorized in the Tucker format (i.e., higher-order singular value decomposition (SVD)) \cite{kolda2009tensor}:
\begin{equation}
\mathcal{C}(\boldsymbol{X})=\boldsymbol{\mathcal{S}}\times_1\boldsymbol{U}_1\times_2\boldsymbol{U}_1\times_3\boldsymbol{U}_2\times_4\boldsymbol{U}_2,
\end{equation}
where $\boldsymbol{\mathcal{S}}\in\mathbb{R}^{N\times N\times T\times T}$ is the core tensor (consisting of singular values \cite{semerci2014tensor, lu2019tensor}), while $\boldsymbol{U}_1\in\mathbb{R}^{N\times N}$ and $\boldsymbol{U}_2\in\mathbb{R}^{T\times T}$ are unitary matrices. The notation $\times_k,\,\forall k\in\{1,2,3,4\}$ represents the mode-$k$ product between tensor and matrix \cite{kolda2009tensor}. The circulant tensor nuclear norm is defined as
\begin{equation}
\|\mathcal{C}(\boldsymbol{X})\|_{*}=\sum_{n=1}^{N}\sum_{t=1}^{T}s_{n,n,t,t},
\end{equation}
where $s_{n,n,t,t}$ is the $(n,n,t,t)$-th entry of the core tensor $\boldsymbol{\mathcal{S}}$.
\end{definition}


Recall that the temporal regularization in the univariate LCR model works on each time series independently. To jointly characterize the spatial and temporal dependencies for traffic time series, we consider a separable kernel in the LCR model, namely, $\boldsymbol{K}\triangleq\boldsymbol{\ell}_{s}\boldsymbol{\ell}^\top\in\mathbb{R}^{N\times T}$ with the spatial kernel $\boldsymbol{\ell}_{s}=(1,0,\cdots,0)^\top\in\mathbb{R}^{N}$ (i.e., the first column of the $N$-by-$N$ identity matrix). In the case of spatial modeling, $\boldsymbol{\ell}_s$ can also be introduced as the Laplacian kernel in Definition~\ref{laplacian_kernel_def}. The optimization problem of two-dimensional LCR (LCR-2D) can be formulated as follows,


\begin{equation}\label{opt_prob_LCR_2D}
\begin{aligned}
\min_{\boldsymbol{X}}~&\|\mathcal{C}(\boldsymbol{X})\|_{*}+\frac{\gamma}{2}\|\boldsymbol{K}\star\boldsymbol{X}\|_{F}^{2} \\
\text{s.t.}~~&\|\mathcal{P}_{\Omega}(\boldsymbol{X}-\boldsymbol{Y})\|_{F}\leq\epsilon,
\end{aligned}
\end{equation}
where $\|\cdot\|_F$ denotes the Frobenius norm of a matrix. The two-dimensional circular convolution is described in Definition~\ref{2d_circular_conv}. Although $\boldsymbol{\ell}_{s}=(1,0,\cdots,0)^\top\in\mathbb{R}^{N}$ does not provide any spatial dependencies, the nuclear norm of circulant operator on $\boldsymbol{X}$ can achieve implicit spatial modeling. 
If applicable, one can characterize spatial correlations by using the Laplacian kernel.

\begin{definition}[Two-Dimensional Circular Convolution \cite{brigham1988fast, hansen2006deblurring}]\label{2d_circular_conv}
For any matrices $\boldsymbol{X}\in\mathbb{R}^{N\times T}$ and $\boldsymbol{K}\in\mathbb{R}^{\nu_1\times\nu_2}$ with $\nu_1\leq N,\nu_2\leq T$, the circular convolution of two matrices is defined as
\begin{equation}
\boldsymbol{Z}=\boldsymbol{K}\star\boldsymbol{X}\in\mathbb{R}^{N\times T},
\end{equation}
or element-wise,
\begin{equation}
z_{n,t}=\sum_{i=1}^{\nu_1}\sum_{j=1}^{\nu_2}\kappa_{i,j}x_{n-i+1,t-j+1},
\end{equation}
where $n=1,2,\ldots,N$ and $t=1,2,\ldots,T$. $\kappa_{i,j}$ is the $(i,j)$-th entry of $\boldsymbol{K}$. 
\end{definition}

\begin{remark}
In the field of signal processing, the two-dimensional circular convolution also possesses the properties that are associated with two-dimensional DFT. According to the convolution theorem and the Parseval's theorem, we have
\begin{equation}
\|\boldsymbol{K}\star\boldsymbol{X}\|_{F}^{2}=\frac{1}{NT}\|\mathcal{F}(\boldsymbol{K})\circ\mathcal{F}(\boldsymbol{X})\|_{F}^{2},
\end{equation}
where $\mathcal{F}(\cdot)$ denotes the two-dimensional DFT. Typically, two-dimensional DFT can be computed by first transforming each column vector (or row vector) and then each row vector (or column vector) of the matrix \cite{hansen2006deblurring}.
\end{remark}

\subsubsection{ADMM Scheme and Solution Algorithm}

As mentioned above, LCR-2D in the multivariate case is a convex problem that can be resolved by the ADMM. Following Eq.~\eqref{admm_scheme}, the ADMM scheme is given by
\begin{equation}\label{2d_admm_scheme}
\left\{
\begin{aligned}
\boldsymbol{X}:&=\argmin_{\boldsymbol{X}}\mathcal{L}(\boldsymbol{X},\boldsymbol{Z},\boldsymbol{W}), \\
\displaystyle
\boldsymbol{Z}:&=\argmin_{\boldsymbol{Z}}\mathcal{L}(\boldsymbol{X},\boldsymbol{Z},\boldsymbol{W}) \\
&=\frac{1}{\lambda+\eta}\mathcal{P}_{\Omega}(\lambda\boldsymbol{X}+\boldsymbol{W}+\eta\boldsymbol{Y})+\frac{1}{\lambda}\mathcal{P}_{\Omega}^\perp(\lambda\boldsymbol{X}+\boldsymbol{W}), \\
\boldsymbol{W}:&=\boldsymbol{W}+\lambda(\boldsymbol{X}-\boldsymbol{Z}),
\end{aligned}\right.
\end{equation}
where $\mathcal{L}(\boldsymbol{X},\boldsymbol{Z},\boldsymbol{W})$ is the augmented Lagrangian function:
\begin{equation}
\begin{aligned}
\mathcal{L}(\boldsymbol{X},\boldsymbol{Z},\boldsymbol{W})=&\|\mathcal{C}(\boldsymbol{X})\|_{*}+\frac{\gamma}{2}\|\boldsymbol{K}\star\boldsymbol{X}\|_{F}^{2}\\
&+\frac{\lambda}{2}\|\boldsymbol{X}-\boldsymbol{Z}\|_{F}^{2}+\langle\boldsymbol{W},\boldsymbol{X}-\boldsymbol{Z}\rangle \\
&+\frac{\eta}{2}\|\mathcal{P}_{\Omega}(\boldsymbol{Z}-\boldsymbol{Y})\|_{F}^{2}.
\end{aligned}
\end{equation}


Although the nuclear norm of circulant tensor in Definition~\ref{circulant_tensor_nuclear_norm_def} is more complicated than the nuclear norm of circulant matrix (see Lemma~\ref{convolution_nuclear_norm}), two-dimensional DFT allows one to find the solution to $\boldsymbol{X}$, if not mentioning the difficulty of obtaining a unique decomposition of circulant tensor. In the frequency domain, it takes
\begin{equation}
\begin{aligned}
\hat{\boldsymbol{X}}:=\argmin_{\hat{\boldsymbol{X}}}~&\|\hat{\boldsymbol{X}}\|_{1}+\frac{\gamma}{2NT}\|\hat{\boldsymbol{K}}\circ\hat{\boldsymbol{X}}\|_{F}^{2}\\
&+\frac{\lambda}{2NT}\|\hat{\boldsymbol{X}}-\hat{\boldsymbol{Z}}+\hat{\boldsymbol{W}}/\lambda\|_{F}^{2},
\end{aligned}
\end{equation}
where $\{\hat{\boldsymbol{K}},\hat{\boldsymbol{X}},\hat{\boldsymbol{Z}},\hat{\boldsymbol{W}}\}$ refers to $\{{\boldsymbol{K}},{\boldsymbol{X}},{\boldsymbol{Z}},{\boldsymbol{W}}\}$ in the frequency domain. Without loss of generality, this subproblem for $\ell_1$-norm minimization in complex space can also be solved by the shrinkage operator in Eq.~\eqref{prox_L1_norm_complex}. Although Definition~\ref{circulant_tensor_nuclear_norm_def} describes the formula of higher-order SVD that leads to the nuclear norm, the computation of the circulant tensor nuclear norm minimization is actually converted into an $\ell_1$-norm minimization with two-dimensional FFT. Algorithm~\ref{LCR_2D} summarizes the whole scheme of LCR-2D.

\begin{algorithm}
\caption{Two-Dimensional Laplacian Convolutional Representation (LCR-2D)}
\label{LCR_2D}
\begin{algorithmic}[1]
\renewcommand{\algorithmicrequire}{\textbf{Input:}}
\renewcommand{\algorithmicensure}{\textbf{Output:}}
\REQUIRE Data $\boldsymbol{Y}\in\mathbb{R}^{N\times T}$ with observed index set $\Omega$, Laplacian kernel size $\tau\in\mathbb{Z}^{+}$, and hyperparameters $\{\gamma,\lambda,\eta\}$.
\ENSURE Reconstructed matrix $\boldsymbol{X}\in\mathbb{R}^{N\times T}$.
\STATE Initialize $\{\boldsymbol{X}_{0},\boldsymbol{Z}_{0},\boldsymbol{W}_{0}\}$.
\STATE Construct the Laplacian kernel $\boldsymbol{\ell}\in\mathbb{R}^{T}$ with $\tau$. 
\STATE Construct the spatial kernel $\boldsymbol{\ell}_s=(1,0,\cdots,0)\in\mathbb{R}^N$ (or $\boldsymbol{\ell}_s\in\mathbb{R}^{N}$ with $\tau_s$) and build up a separable kernel $\boldsymbol{K}\triangleq\boldsymbol{\ell}_s\boldsymbol{\ell}^\top$.
\FOR {$i=0$ to maximum iteration}
\STATE Perform FFT on $\{\boldsymbol{Z}_{i},\boldsymbol{W}_{i}\}$.
\STATE Compute $\hat{\boldsymbol{X}}$ by referring to the shrinkage in Eq.~\eqref{prox_L1_norm_complex}.
\STATE Compute $\boldsymbol{X}_{i+1}$ by $\boldsymbol{X}_{i+1}=\mathcal{F}^{-1}(\hat{\boldsymbol{X}})$.
\STATE Compute $\boldsymbol{Z}_{i+1}$ by Eq.~\eqref{2d_admm_scheme}.
\STATE Compute $\boldsymbol{W}_{i+1}$ by Eq.~\eqref{2d_admm_scheme}.
\ENDFOR
\end{algorithmic}
\end{algorithm}

\section{Univariate Traffic Time Series Imputation}\label{univariate_experiment}

This section evaluates the reconstruction of univariate time series from partial observations with LCR. We focus on understanding how well the reconstructed time series preserve the global and local trends in the imputation task. Experiments are conducted on both traffic speed time series (weak periodicity and strong noises) and traffic volume time series (strong periodicity) collected through dual-loop detectors on the highway network of Portland, USA.\footnote{\url{https://portal.its.pdx.edu/}} Specifically, the datasets include: (\textbf{Traffic speed}) The speed observations with 15-min time resolution (i.e., 96 expected data samples per day) over three days (i.e., of length 288). (\textbf{Traffic volume}) The volume observations have the same time resolution as the traffic speed. In particular, we consider the imputation scenarios on fully observed, 20\%, 10\%, and 5\% observed data, respectively. To generate partially observed data, we randomly mask a certain number of data as missing values. Since both traffic speed and volume time series show cyclical patterns, the start data points and the end data points can be connected by circulant matrix and circular convolution without flipping operation in Remark~\ref{flip_vec_remark}.

In the following experiments, to evaluate the imputation performance, we use the mean absolute percentage error (MAPE) and the root mean square error (RMSE):
\begin{equation*}
\begin{aligned}
\text{MAPE}=\frac{1}{n}\sum_{i=1}^{n}\frac{|y_i-\hat{y}_i|}{y_i},\quad
 \text{RMSE}=\sqrt{\frac{1}{n}\sum_{i=1}^{n}(y_i-\hat{y}_i)^2},
\end{aligned}
\end{equation*}
where $n$ is the total number of estimated values, and $y_i$ and $\hat{y}_i$ are the actual value and its estimation, respectively.

\subsection{Traffic Speed}


\begin{figure*}[!ht]
\centering
\subfigure[Fully observed.]{
    \centering
    \includegraphics[width = 0.315\textwidth]{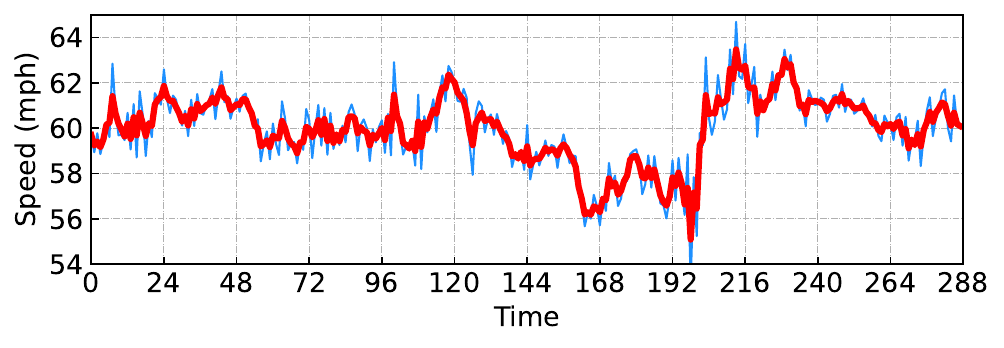}\label{speeds_0}
}
\subfigure[80\% missing values. Here, $\text{MAPE} = 1.42\%$.]{
    \centering
    \includegraphics[width = 0.315\textwidth]{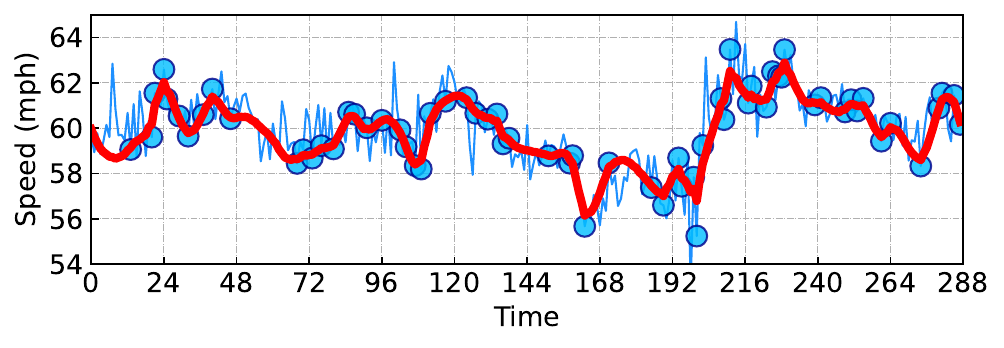}\label{speeds_80}
}
\subfigure[90\% missing values. Here, $\text{MAPE} = 1.69\%$.]{
    \centering
    \includegraphics[width = 0.315\textwidth]{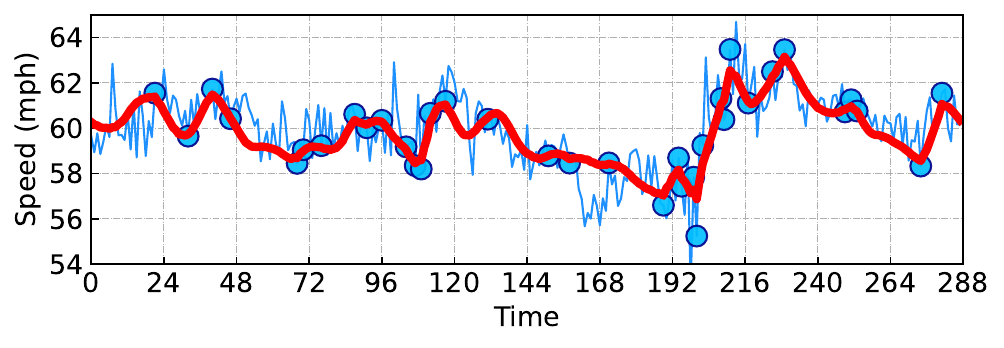}\label{speeds_90}
}
\caption{Univariate traffic time series imputation on the freeway traffic speed time series. The blue curve represents the ground truth time series, while the red curve refers to the reconstructed time series produced by LCR. Here, partial observations are illustrated as blue circles.}
\label{univariate_speed_laplacian_imputation}
\end{figure*}

Fig.~\ref{speeds_0} demonstrates the reconstructed time series by LCR on the fully observed time series, in which the data noises can be smoothed out due to the existence of temporal regularization and the relaxation of observation constraint $\|\mathcal{P}_{\Omega}(\boldsymbol{x}-\boldsymbol{y})\|_{2}\leq \epsilon$ in which $\boldsymbol{y}$ and $\boldsymbol{x}$ are the partial observations and the reconstructed time series, respectively. Fig.~\ref{speeds_80} and \ref{speeds_90} show the imputation performance by LCR with partial observations. Of these results, the reconstructed time series can accurately approximate both partial observations and missing values while preserving the trends of the ground truth time series.

\begin{figure}[!ht]
\centering
\subfigure[CircNNM (red curve). Here, $\text{MAPE} = 2.47\%$.]{
    \centering
    \includegraphics[width = 0.35\textwidth]{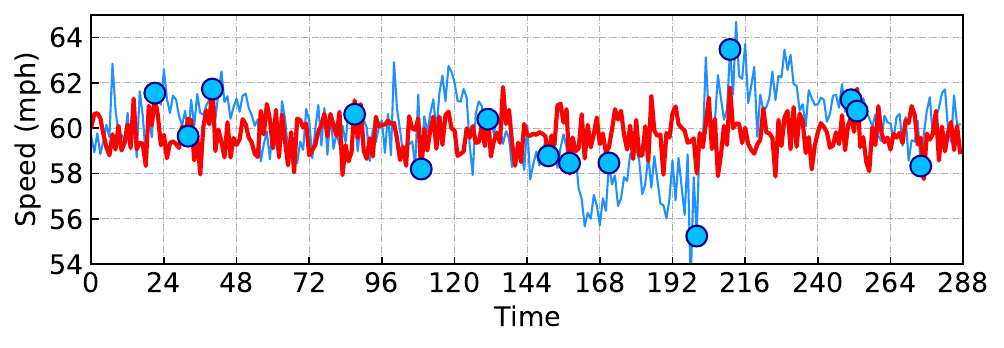}
}
\subfigure[ConvNNM (red curve). Here, $\text{MAPE} = 2.33\%$.]{
    \centering
    \includegraphics[width = 0.35\textwidth]{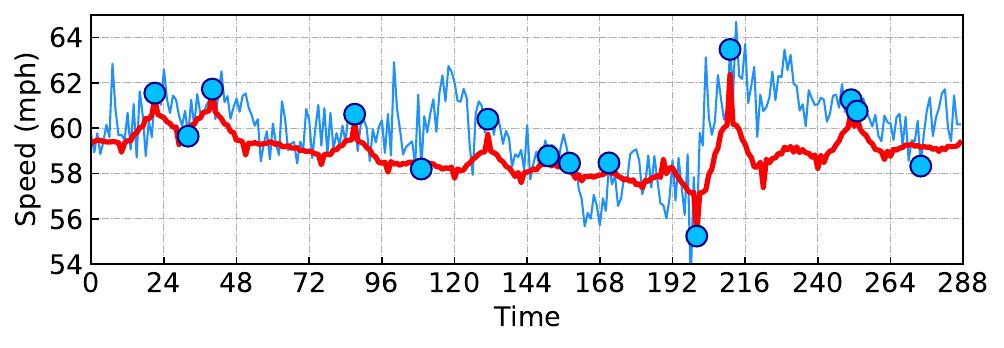}
}
\subfigure[ConvNNM+ (red curve). Here, $\text{MAPE} = 2.30\%$.]{
    \centering
    \includegraphics[width = 0.35\textwidth]{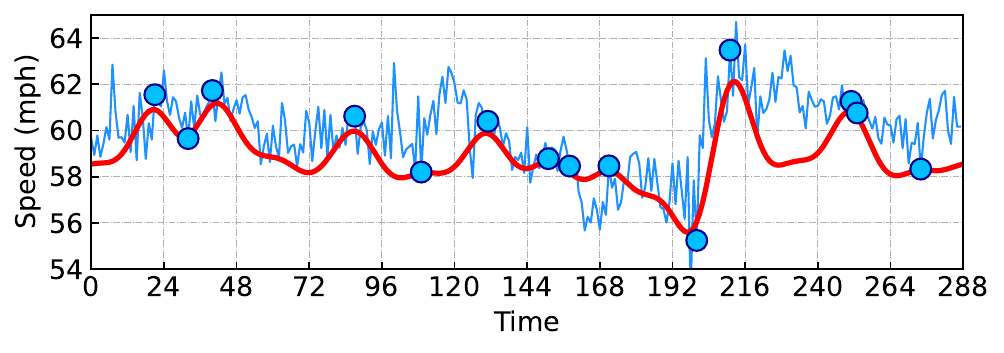}
}
\subfigure[LCR (red curve) with $\tau=2,\gamma=5\lambda$. Here, $\text{MAPE} = 2.13\%$.]{
    \centering
    \includegraphics[width = 0.35\textwidth]{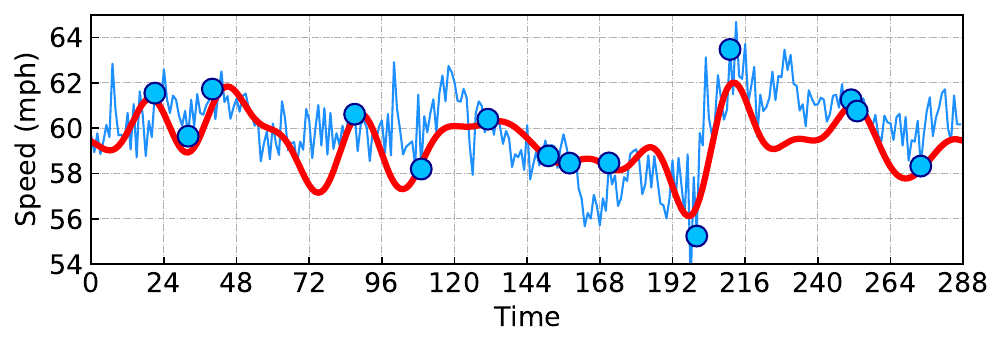}\label{speeds_tau_2_gamma_2}
}
\caption{Univariate traffic time series imputation on the freeway traffic speed time series. In this case, we mask 95\% observations as missing values and only have 14 speed observations for training the model.}
\label{univariate_speed_imputation}
\end{figure}

Next, we test a more challenging scenario in which we aim to reconstruct 95\% missing values from 5\% observations (i.e., reconstructing 274 missing values from only 14 data samples). We compare the time series imputation of LCR with the following baseline models: (i) CircNNM \cite{liu2022recovery} (equivalent to LCR without temporal regularization), (ii) ConvNNM \cite{liu2022recovery}, and (iii) ConvNNM+
(i.e., ConvNNM with temporal regularization) such that
\begin{equation}
\begin{aligned}
\min_{\boldsymbol{x}}~&\|\mathcal{C}_{\tilde{\tau}}(\boldsymbol{x})\|_{*}+\frac{\gamma}{2}\|\boldsymbol{\ell}\star\boldsymbol{x}\|_{2}^{2} \\
\text{s.t.}~~&\|\mathcal{P}_{\Omega}(\boldsymbol{x}-\boldsymbol{y})\|_{2}\leq\epsilon.
\end{aligned}
\end{equation}

Essentially, comparing the proposed LCR model with these baseline models allows one to highlight 1) the importance of global and local trends modeling in LCR, and 2) the fast implementation of LCR via the use of FFT. The imputation results on the given time series are shown in Fig.~\ref{univariate_speed_imputation}, we can summarize the following findings: 
\begin{itemize}
\item The reconstructed time series by CircNNM shows high fluctuations due to the lack of local trend modeling. As a result, the reconstructed time series fails to reproduce trends of the ground truth at such a high missing rate.
\item ConvNNM performs better than CircNNM. The reconstructed time series fits the observed values well, but the trend does not match the ground truth perfectly. Unlike CircNNM, ConvNNM cannot employ a fast implementation via FFT. As a result, ConvNNM is not well-suited to large problems (also see Fig.~\ref{empirical_time_complexity_curve}).
\item ConvNNM+ outperforms ConvNNM. This demonstrates the significance of Laplacian kernel for time series modeling. However, both ConvNNM and ConvNNM+ require implementing the singular value thresholding on $T$-by-$\tilde{\tau}$ convolution matrices.
\end{itemize}

As shown in Fig.~\ref{speeds_tau_2_gamma_2}, the reconstructed time series by LCR demonstrates consistent global and local trends with the trends of ground truth time series, and our LCR model clearly outperforms the baseline models. By comparing CircNNM with LCR, the imputation results emphasize the importance of temporal regularization $\mathcal{R}_{\tau}(\boldsymbol{x})$.

\subsection{Traffic Volume}

\begin{figure*}[!ht]
\centering
\subfigure[CircNNM. Here, $\text{MAPE} = 36.31\%$.]{
    \centering
    \includegraphics[width = 0.3\textwidth]{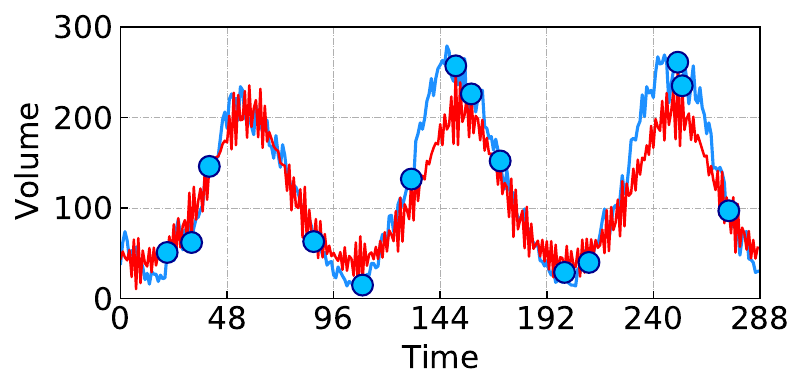}
}
\subfigure[ConvNNM. Here, $\text{MAPE} = 33.18\%$.]{
    \centering
    \includegraphics[width = 0.3\textwidth]{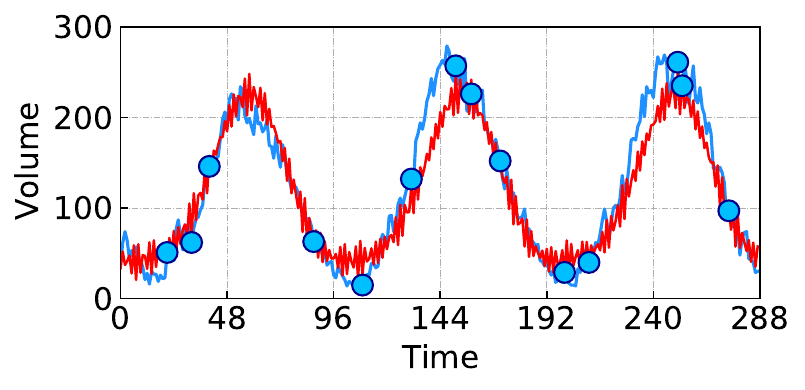}
}
\subfigure[LCR with $\tau=2,\gamma=5\lambda$. Here, $\text{MAPE} = 19.59\%$.]{
    \centering
    \includegraphics[width = 0.3\textwidth]{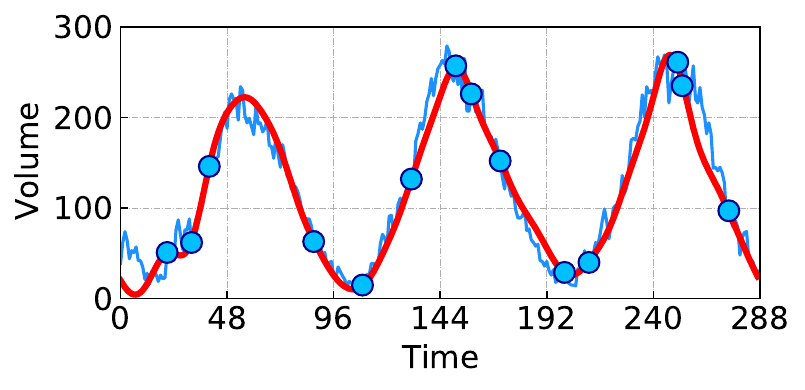}
}
\caption{Univariate time series imputation on the freeway traffic volume time series. In this case, we randomly remove 95\% observations as missing values, and we only have 14 volume observations for the reconstruction.}
\label{univariate_volume_imputation}
\end{figure*}

In Fig.~\ref{univariate_volume_imputation}, the traffic volume time series is characterized by a strong daily rhythm. The time series shows relatively low traffic volume at night, and the traffic volume reaches a peak during rush hours, implying typical travelers' behavioral rhythms. In this case, the time series possesses strong seasonality and three remarkable peaks over three days.

The task is reconstructing 95\% missing values from 5\% observations. As can be seen, due to the lack of explicit temporal modeling, both CircNNM and ConvNNM models cannot produce time series as smooth as LCR. Observing the reconstructed time series, it is clear that LCR produces more accurate reconstruction results than both CircNNM and ConvNNM. Yet, in contrast to the traffic speed time series, due to the strong seasonality (i.e., daily rhythm in traffic flow) in the traffic volume time series, both CircNNM and ConvNNM can produce reasonable time series, which seems to be consistent with the results in \cite{liu2022time}.

\section{Multivariate Traffic Time Series Imputation}\label{multivariate_experiment}

In this section, we consider some real-world multivariate traffic time series imputation scenarios, including speed field reconstruction of vehicular traffic flow and large-scale traffic speed data imputation. The experiments and evaluation of LCR are expected to demonstrate the efficiency of global and local trends modeling on traffic time series with spatiotemporal settings. 

\subsection{Speed Field Reconstruction}

Speed field reconstruction is a critical problem in vehicular traffic flow modeling as the data collection process is often far from ideal \cite{treiber2013traffic}. The task is to reconstruct the traffic speed distribution along a road segment during a period using the motion states (i.e., time, location, and speed) of a small portion of vehicles \cite{de2008traffic}, e.g., to estimate the traffic congestion by using only the motion states of taxis.

\subsubsection{Datasets}

Traffic states such as movement speeds measure the status of vehicular traffic flow, which can be extracted from high-resolution vehicle trajectories \cite{treiber2013traffic}. In the literature, there are several classical traffic flow datasets, such as HighD\footnote{Highway Drone (HighD) dataset is available at \url{https://www.highd-dataset.com/}.} and CitySim\footnote{A Drone-Based Vehicle Trajectory Dataset for Safety-Oriented Research and Digital Twins (CitySim) is available at \url{https://github.com/ozheng1993/UCF-SST-CitySim-Dataset}.} collected through drones across a fraction of roads \cite{krajewski2018highd,zheng_2023citysim}. The drones can gather the movement trajectories of vehicles for driving behavior modeling and traffic state analysis. In terms of the HighD dataset, there are 60 video recordings available from several German highway sections, and each covers trajectory data on the 420-meter road segment with multiple lanes. We select the video \#46 for experiments and the resultant speed field tensor is of size $142\times 595\times 3$, i.e., $142$ discrete locations (3-meter spatial resolution), $595$ time steps (2-second time resolution), and $3$ lanes, see Fig.~\ref{speed_matrix_full_46_lane1}, \ref{speed_matrix_full_46_lane2}, and \ref{speed_matrix_full_46_lane3}. In terms of the CitySim dataset, it contains video recordings of vehicular trajectory data collected from various road infrastructures, including freeways, expressways, and intersections. We select the trajectory data of one freeway for experiments. The resultant speed field tensor is of size $126\times 442\times 3$, i.e., 126 discrete locations (5-meter spatial resolution), 442 time steps (2-second time resolution), and 3 lanes, see Fig.~\ref{speed_matrix_full_lane1}, \ref{speed_matrix_full_lane2}, and \ref{speed_matrix_full_lane3}. In the experiment, we randomly mask a certain fraction of trajectories (e.g., 30\%, 50\%, and 70\%) and then construct the speed field that shows partial observations.

\subsubsection{Baseline Models}

For comparison, we consider some low-rank matrix/tensor completion models, including LRMC \cite{cai2010singular}, Hankel tensor factorization (HTF \cite{yokota2018missing} in the form of CP factorization), High-accuracy low-rank tensor completion (HaLRTC \cite{liu2013tensor}), and low-rank tensor completion with truncated nuclear norm (LRTC-TNN \cite{chen2020nonconvex}). To highlight the advantage of LCR-2D, we also choose the following baseline models:
\begin{itemize}
\item LCR$_N$: We consider to implement LCR over $N$ univariate time series problems independently, i.e.,
\begin{equation}
\begin{aligned}
\min_{\boldsymbol{X}}~&\sum_{n=1}^{N}\|\mathcal{C}(\boldsymbol{x}_{n})\|_{*}+\frac{\gamma}{2}\sum_{n=1}^{N}\|\boldsymbol{\ell}\star\boldsymbol{x}_{n}\|_{2}^{2} \\
\text{s.t.}~~&\|\mathcal{P}_{\Omega}(\boldsymbol{X}-\boldsymbol{Y})\|_{F}\leq\epsilon, \\
\end{aligned}
\end{equation}
where $\boldsymbol{x}_n\in\mathbb{R}^{T},\,n=1,2,\ldots, N$ are the univariate time series of $\boldsymbol{X}\in\mathbb{R}^{N\times T}$.
\item CTNNM: Circulant tensor nuclear norm minimization whose objective function is specified as $\|\mathcal{C}(\boldsymbol{X})\|_{*}$ (i.e., the special case of LCR-2D without regularization term).
\item Quadratic variation completion (QVC) such that
\begin{equation}
\begin{aligned}
\min_{\boldsymbol{X}}~&\frac{\gamma}{2}\operatorname{tr}(\boldsymbol{X}\tilde{\boldsymbol{L}}\boldsymbol{X}^\top) \\
\text{s.t.}~~&\|\mathcal{P}_{\Omega}(\boldsymbol{X}-\boldsymbol{Y})\|_{F}\leq\epsilon,
\end{aligned}
\end{equation}
with the Laplacian matrix $\tilde{\boldsymbol{L}}$ referring to Eq.~\eqref{QV_laplacian_matrix}, and $\operatorname{tr}(\cdot)$ denotes the trace of matrix. In the meanwhile, we let the spatial kernel be $\ell_s=(1,0,\cdots,0,0)\in\mathbb{R}^N$ and the temporal kernel be parameterized by $\tau$ in Definition~\ref{laplacian_kernel_def}, referring to Laplacian kernelized completion (LKC, i.e., the special case of LCR-2D without circulant tensor nuclear norm).
\end{itemize}

\subsubsection{Model Setting}

To eliminate correlations between the start data points and the end data points in our model, we introduce a $2N$-by-$2T$ block matrix (i.e., with four blocks) which flips the original matrix $\boldsymbol{Y}$, see Fig.~\ref{flip_matrix}. This matrix can be regarded as the input into the LCR-2D model. The model results are constructed by averaging the blocks according to the flipping operation.



\begin{figure}[ht!]
\centering
\includegraphics[width=0.31\textwidth]{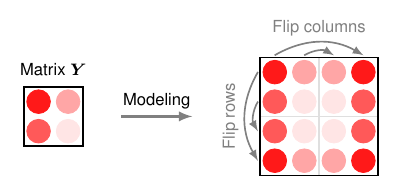}
\caption{Constructing the matrix that flips the original matrix $\boldsymbol{Y}$ along rows and columns simultaneously. This operation can prevent the LCR-2D model from misleading values on the border rows and columns.}
\label{flip_matrix}
\end{figure}

In this case, we intend to highlight the importance of global/local trend modeling on the speed field. Since each speed field dataset includes lane-level trajectories, we evaluate LCR-2D on each lane-level speed field independently. The default hyperparameter $\eta$ is set as $10^2\lambda$. On the HighD dataset, we set $\lambda=10^{-3}NT$ and $\gamma=\lambda$. We validate the kernel size as $\tau=1$ on the 30\%-trajectory speed field and $\tau=2$ on the 50\%-/70\%-trajectory speed field. On the CitySim dataset, we set $\lambda=10^{-4}NT$, $\gamma=\lambda$, and $\tau=3$ after the numerical evaluation.

\subsubsection{Results}

\begin{table*}[!ht]
\caption{Reconstruction performance (in MAPE (\%)/RMSE) achieved by LCR-2D and baseline models on both speed field datasets. Note that the best results are emphasized in bold fonts, and we underline the second-best results.}
\label{LCR_reconstruction_results}
\centering
\begin{tabular}{l|c|ccccccccc}
\toprule
Dataset & Rate & LCR-2D & LCR$_N$ & CTNNM & QVC & LKC & LRMC \cite{cai2010singular} & HTF \cite{yokota2018missing} & HaLRTC \cite{liu2013tensor} & LRTC-TNN \cite{chen2020nonconvex} \\
\midrule
\multirow{3}{*}{HighD} 
& 30\% & \textbf{3.57}/\textbf{1.41} & 6.79/4.09 & 3.91/\underline{1.46} & 6.79/4.59 & 6.62/4.60 & 7.11/4.61 & 4.38/1.66 & 4.63/1.78 & \underline{3.89}/1.50 \\
& 50\% & \textbf{4.06}/\textbf{1.52} & 6.88/3.49 & 4.57/\underline{1.61} & 7.02/3.97 & 6.06/3.88 & 9.37/5.30 & 4.67/1.63 & 5.63/2.06 & \underline{4.53}/1.63 \\
& 70\% & \textbf{4.73}/\textbf{1.77} & 7.91/3.32 & \underline{5.61}/\underline{1.87} & 8.12/4.03 & 6.52/3.87 & 16.79/8.91 & 5.76/2.17 & 7.94/3.01 & 5.75/2.03 \\
\midrule
\multirow{3}{*}{CitySim} 
& 30\% & \textbf{8.88}/\textbf{2.71} & 10.88/4.24 & \underline{9.23}/2.81 & 10.97/4.30 & 10.90/4.24 & 11.27/4.43 & 9.28/\underline{2.78} & 10.91/4.31 & 10.90/4.26 \\
& 50\% & \textbf{9.08}/\textbf{2.69} & 10.56/3.89 & 9.59/\underline{2.82} & 10.65/3.96 & 10.65/3.92 & 11.42/4.29 & \underline{9.49}/2.83 & 10.66/3.99 & 10.49/3.87 \\
& 70\% & \textbf{9.07}/\textbf{2.66} & 10.35/3.67 & \underline{9.42}/\underline{2.73} & 10.33/3.70 & 10.53/3.71 & 12.96/4.97 & 9.93/2.94 & 11.02/3.92 & 10.46/3.71 \\
\bottomrule
\end{tabular}
\end{table*}

\begin{figure*}[ht!]
\centering
\subfigure[Original speed field (lane \#1).]{
\centering
\includegraphics[width = 0.31\textwidth]{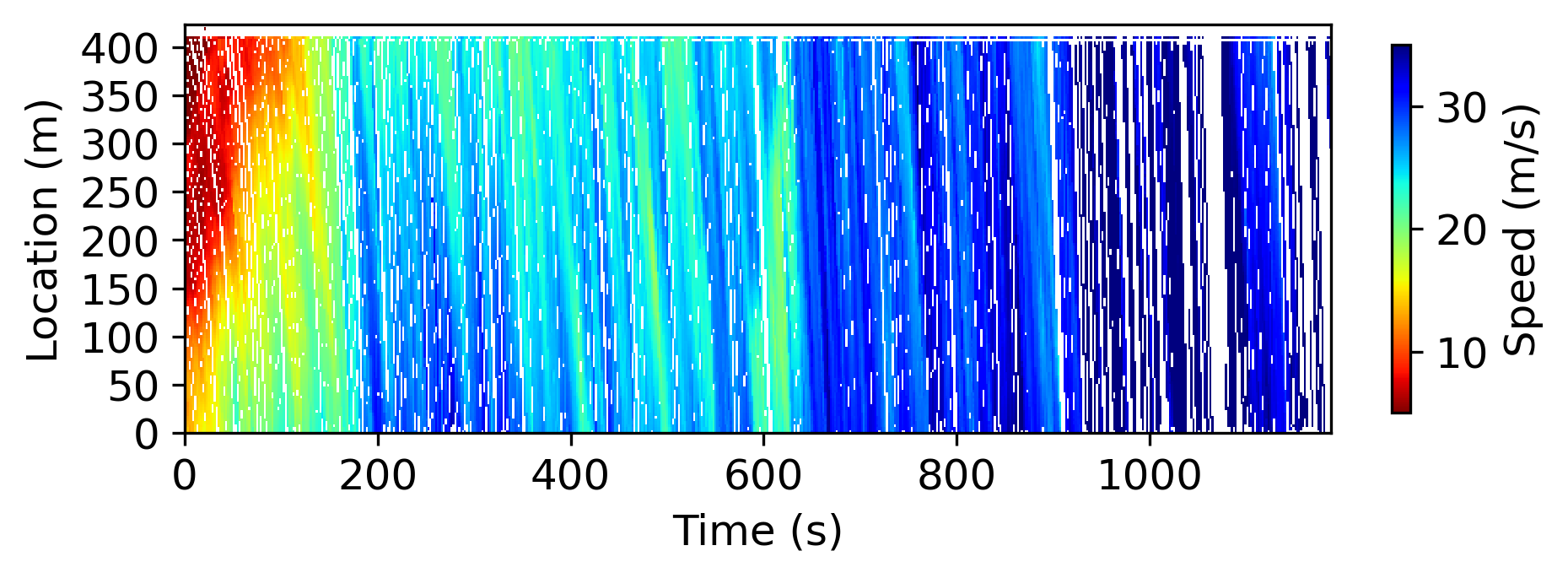}\label{speed_matrix_full_46_lane1}
}
\subfigure[Incomplete speed field (lane \#1).]{
\centering
\includegraphics[width = 0.31\textwidth]{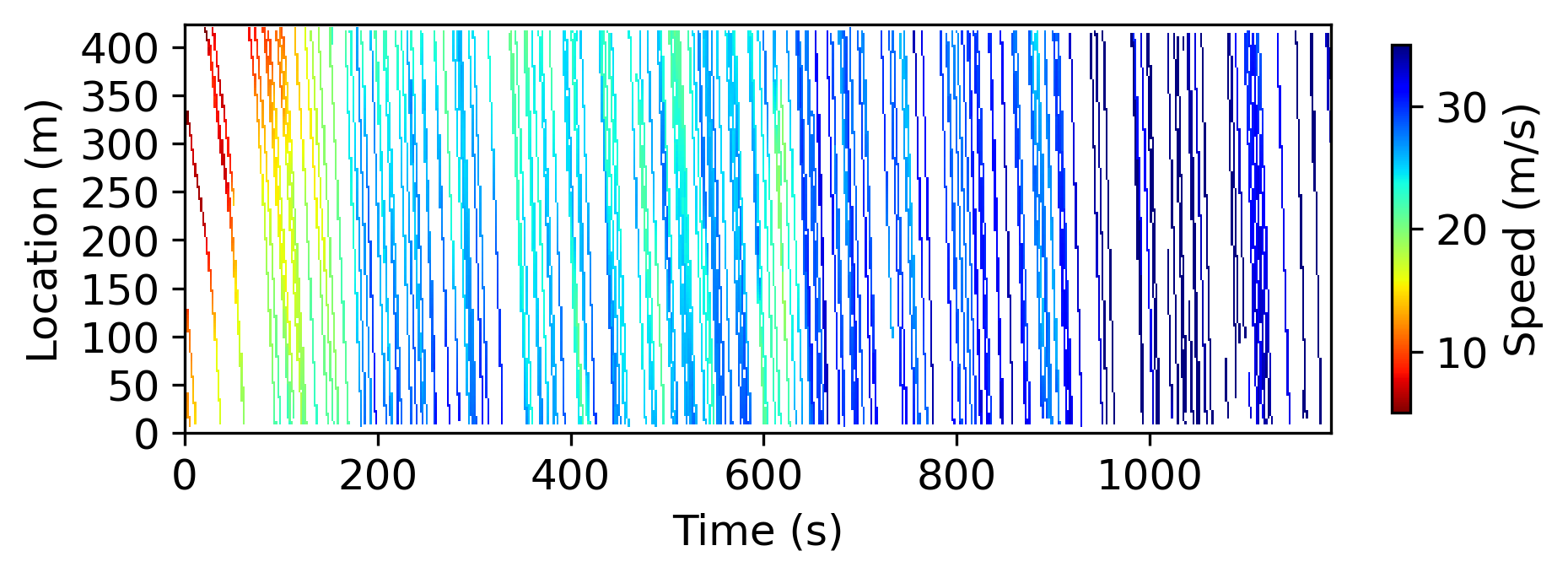}\label{speed_matrix_70_46_lane1}
}
\subfigure[Reconstructed speed field (lane \#1).]{
\centering
\includegraphics[width = 0.31\textwidth]{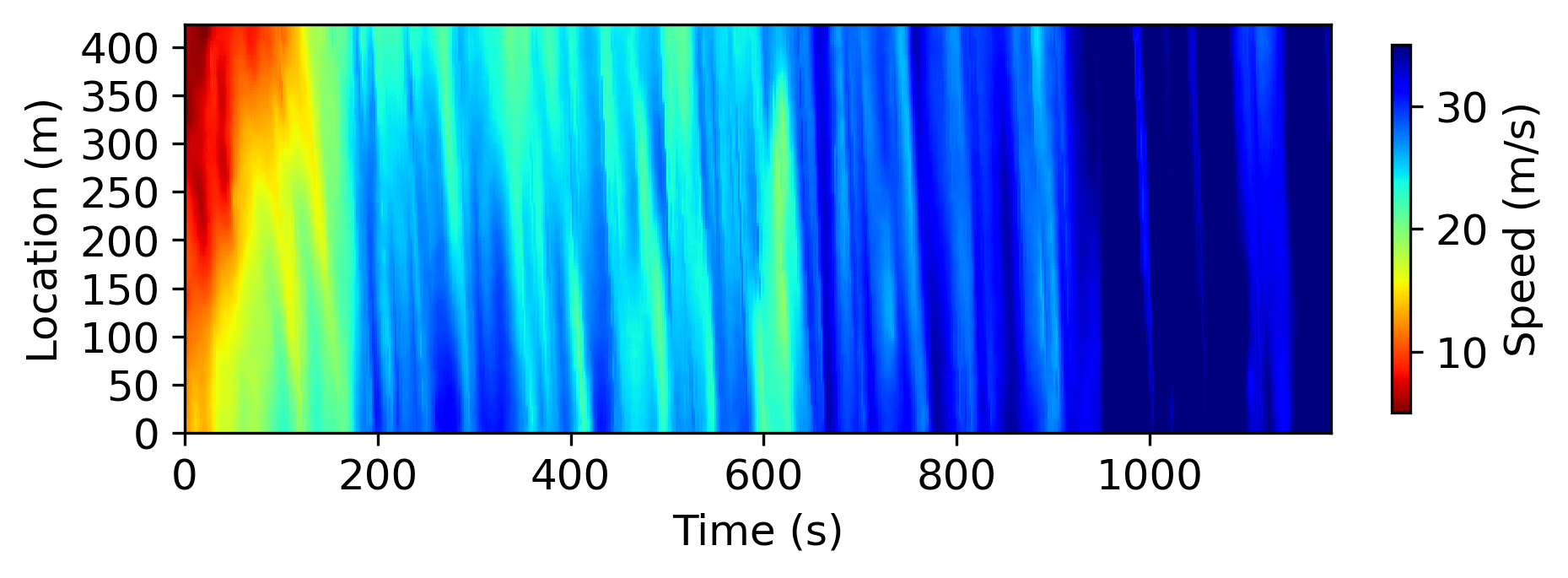}\label{HighD_speed_field_70_LCR_2D_rec_lane1}
}
\subfigure[Original speed field (lane \#2).]{
\centering
\includegraphics[width = 0.31\textwidth]{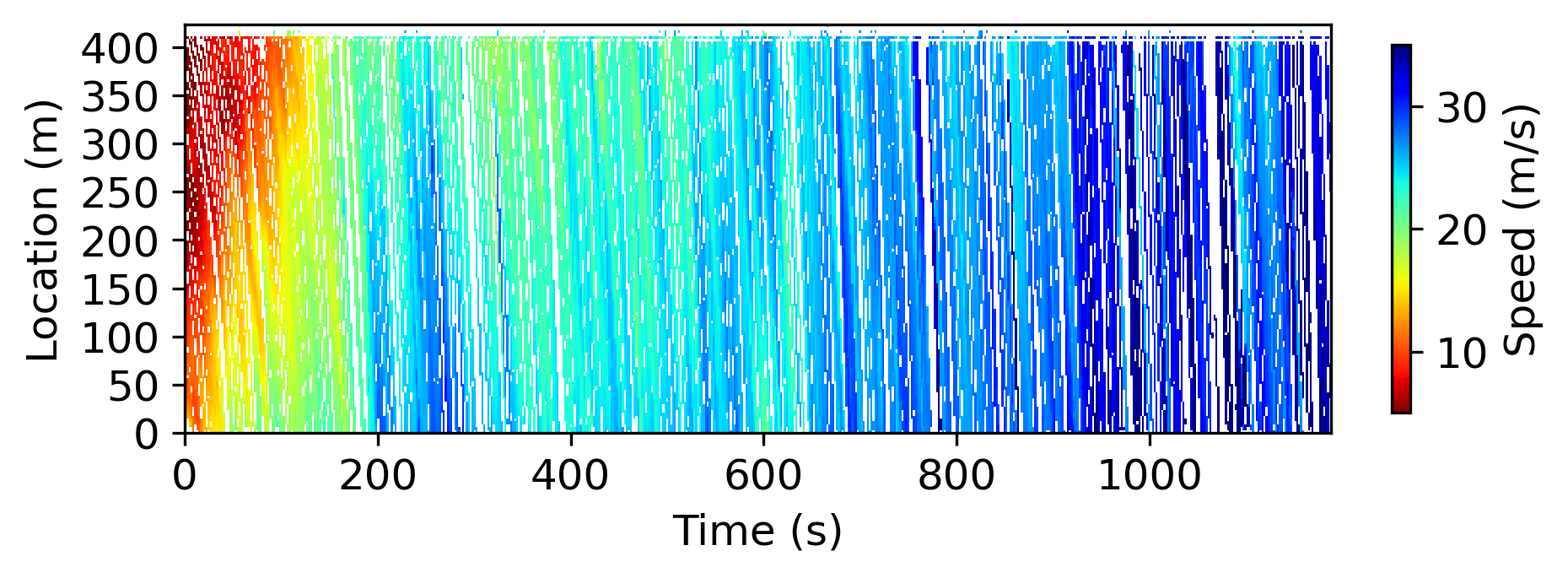}\label{speed_matrix_full_46_lane2}
}
\subfigure[Incomplete speed field (lane \#2).]{
\centering
\includegraphics[width = 0.31\textwidth]{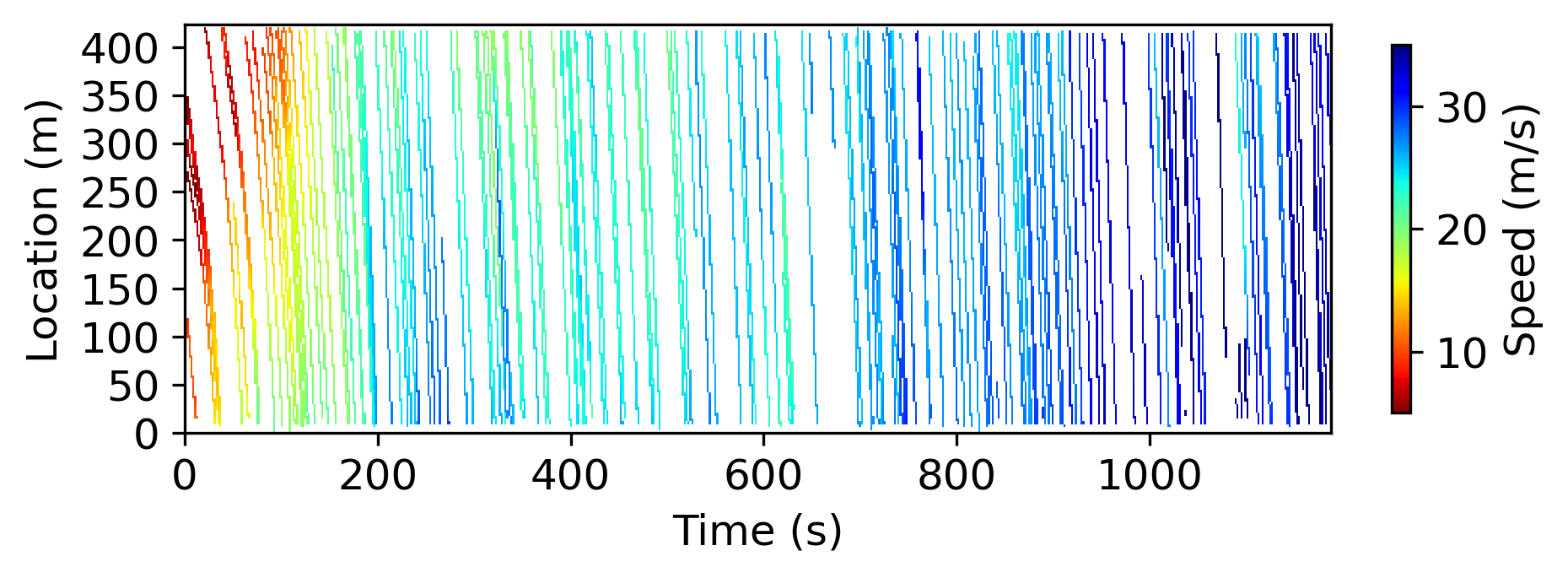}\label{speed_matrix_70_46_lane2}
}
\subfigure[Reconstructed speed field (lane \#2).]{
\centering
\includegraphics[width = 0.31\textwidth]{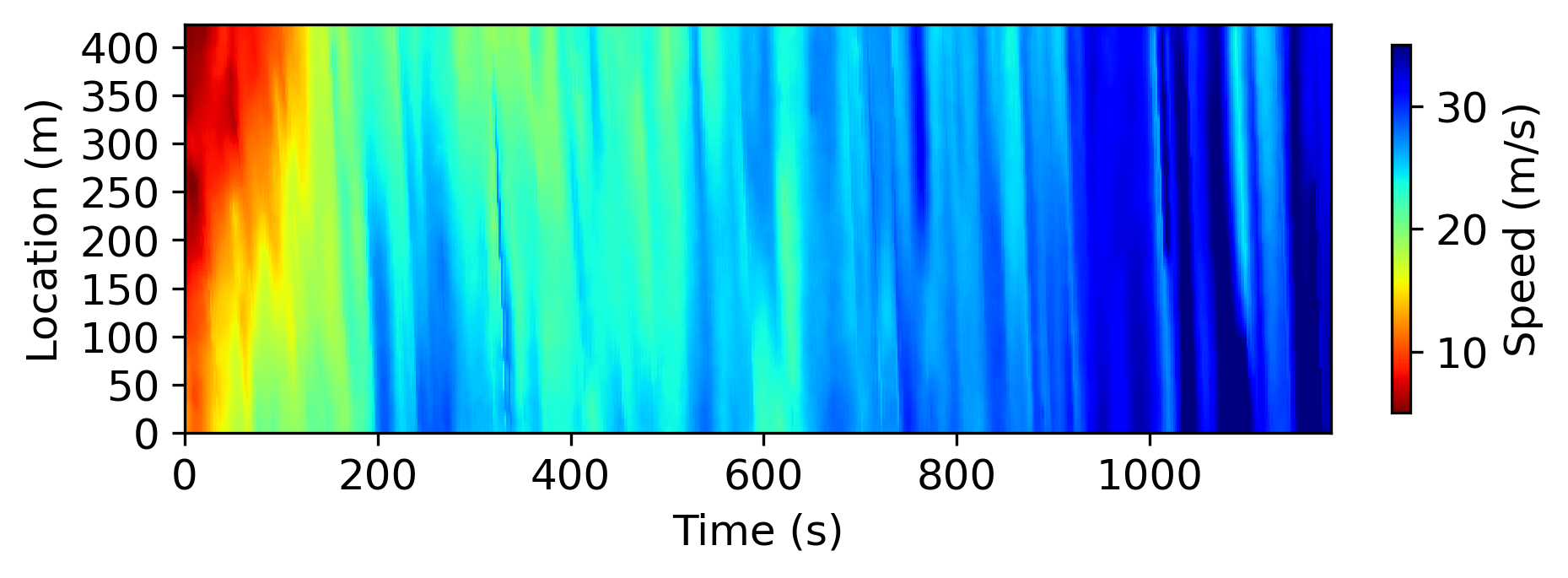}\label{HighD_speed_field_70_LCR_2D_rec_lane2}
}
\subfigure[Original speed field (lane \#3).]{
\centering
\includegraphics[width = 0.31\textwidth]{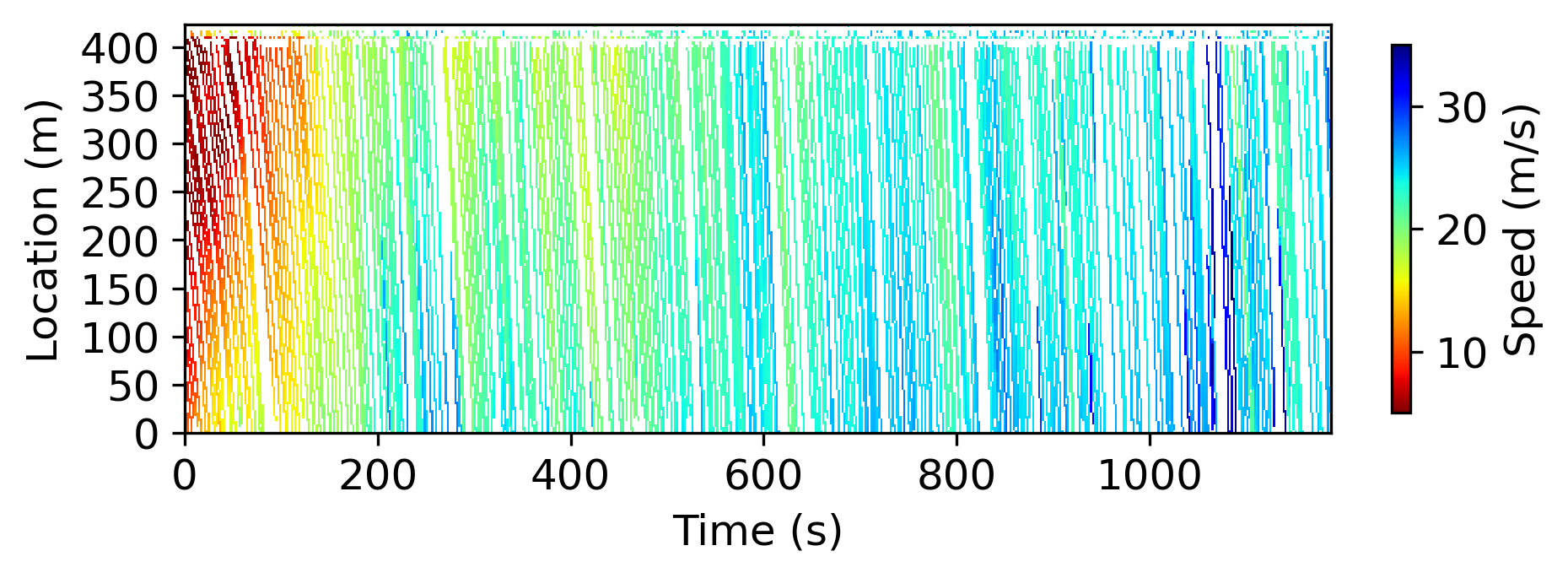}\label{speed_matrix_full_46_lane3}
}
\subfigure[Incomplete speed field (lane \#3).]{
\centering
\includegraphics[width = 0.31\textwidth]{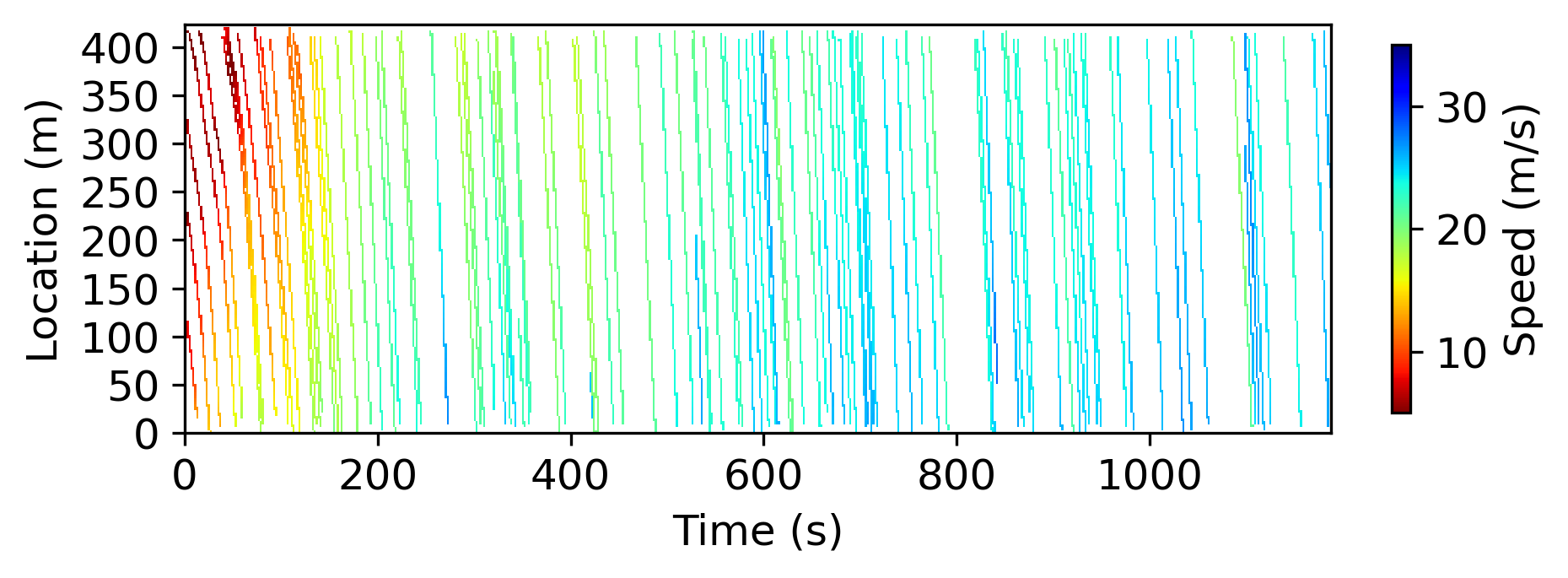}\label{speed_matrix_70_46_lane3}
}
\subfigure[Reconstructed speed field (lane \#3).]{
\centering
\includegraphics[width = 0.31\textwidth]{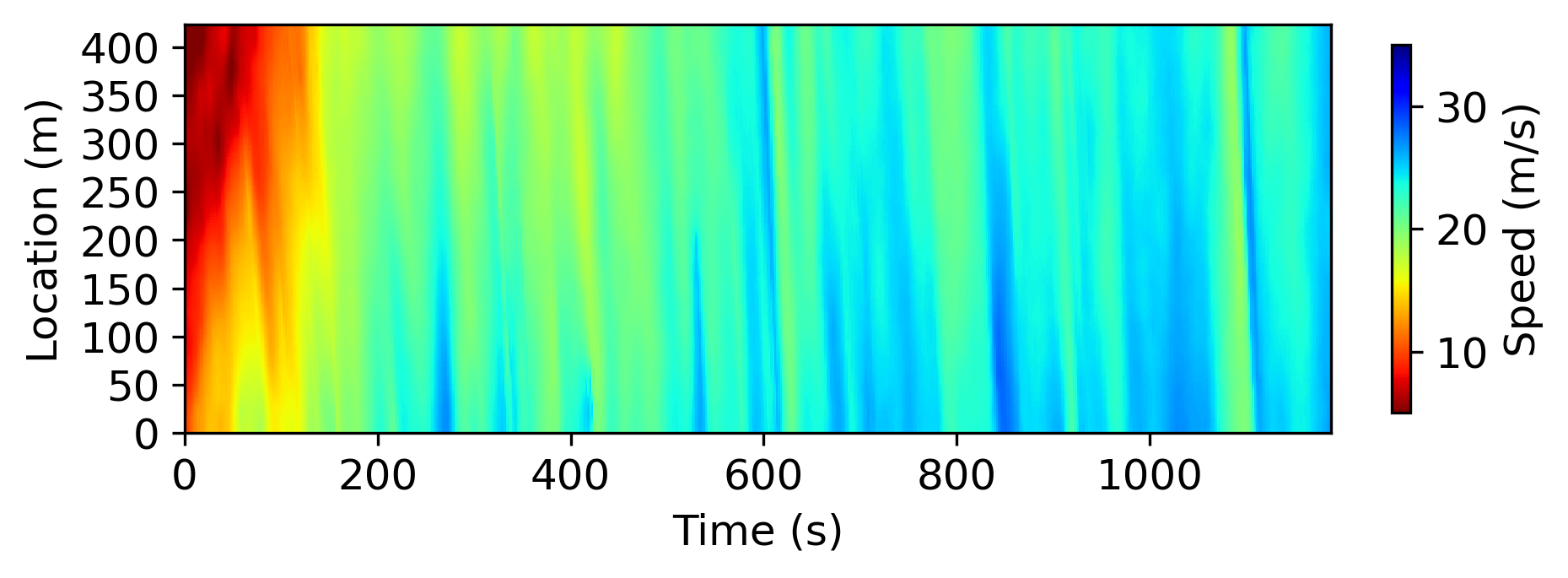}\label{HighD_speed_field_70_LCR_2D_rec_lane3}
}
\caption{Speed field reconstruction achieved by LCR-2D on 70\% masked trajectories of the HighD data. The speed fields with duration of 1,190 seconds are collected from the road segment of length 426 meters.}
\label{HighD_speed_field_reconstruction_70}
\end{figure*}

Table~\ref{LCR_reconstruction_results} gives the speed field reconstruction performance of LCR-2D and some baseline models. We can summarize the following findings: (i) As LCR$_N$ works on the univariate time series independently, spatial correlations of speed fields are totally ignored. The better performance of LCR-2D against LCR$_N$ highlights the importance of spatial modeling for speed field reconstruction. (ii) Recall that CTNNM is a special case of LCR-2D without Laplacian kernels, we empirically demonstrate that LCR-2D outperforms CTNNM on both datasets, implying that the spatiotemporal local trend modeling is of great significance for producing accurate reconstruction of sparse speed fields. (iii) QVC and LKC cannot produce accurate reconstruction of speed fields as the proposed LCR-2D model, and this implies the importance of circulant tensor nuclear norm in LCR-2D. (iv) LCR-2D performs better than a sequence of matrix/tensor completion models on both datasets. This demonstrates that an appropriate spatiotemporal modeling on the low-rank framework is important for speed field reconstruction.

As mentioned above, both circulant tensor nuclear norm and regularization term with Laplacian kernels in the objective function of LCR-2D are of great significance. Fig.~\ref{HighD_speed_field_reconstruction_70} and \ref{CitySim_speed_field_reconstruction_70} show the reconstruction of speed fields with sparse inputs (i.e., on 70\% masked trajectories). It seems that LCR-2D can achieve satisfactory results for speed field reconstruction because of the spatiotemporal correlation of traffic wave \cite{wang2021low} characterized by both global and local trends.

\begin{figure*}[ht!]
\centering
\subfigure[Original speed field (lane \#1).]{
\centering
\includegraphics[width = 0.31\textwidth]{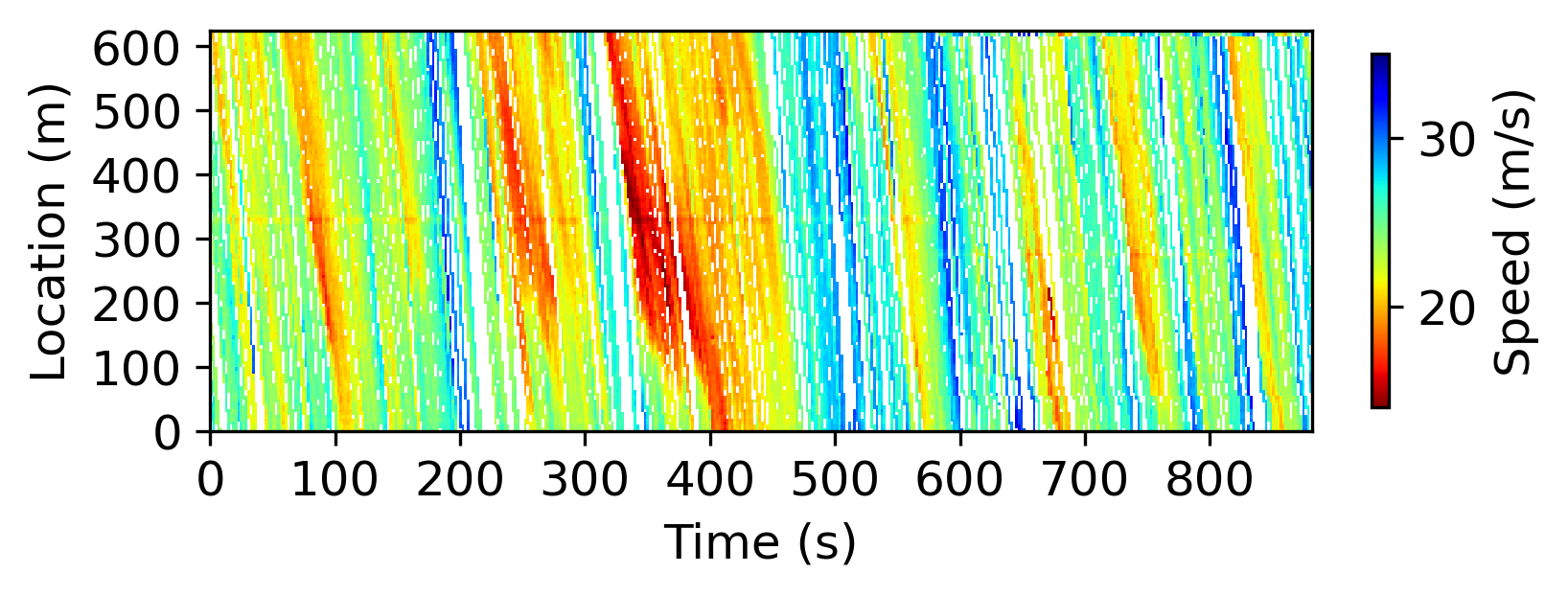}\label{speed_matrix_full_lane1}
}
\subfigure[Incomplete speed field (lane \#1).]{
\centering
\includegraphics[width = 0.31\textwidth]{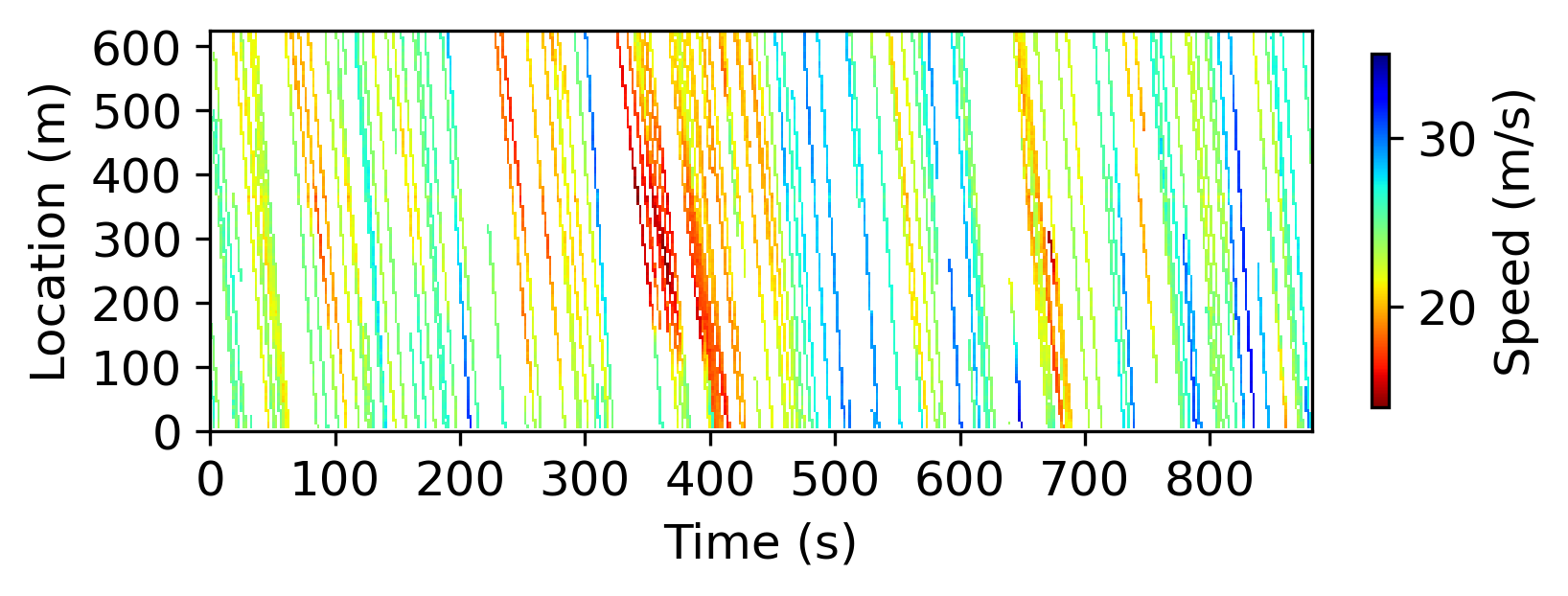}\label{speed_matrix_70_lane1}
}
\subfigure[Reconstructed speed field (lane \#1).]{
\centering
\includegraphics[width = 0.31\textwidth]{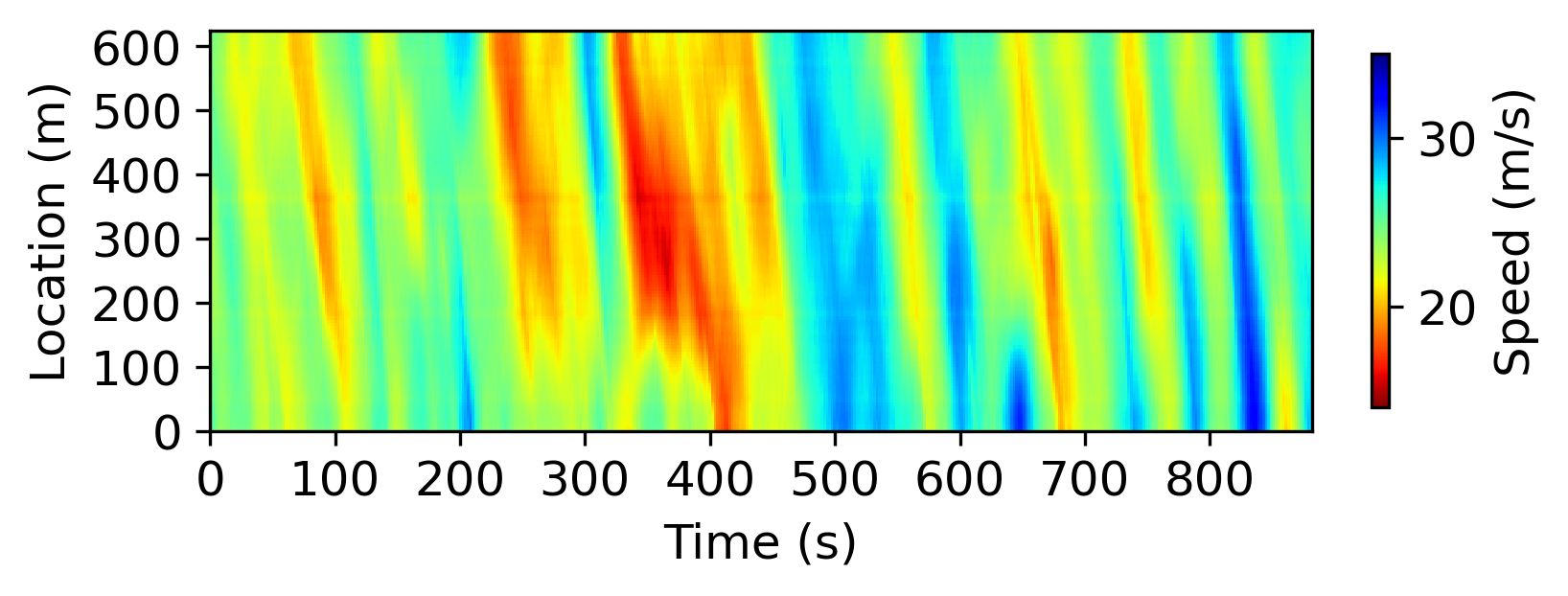}\label{CitySim_speed_field_70_LCR_2D_rec_lane1}
}
\subfigure[Original speed field (lane \#2).]{
\centering
\includegraphics[width = 0.31\textwidth]{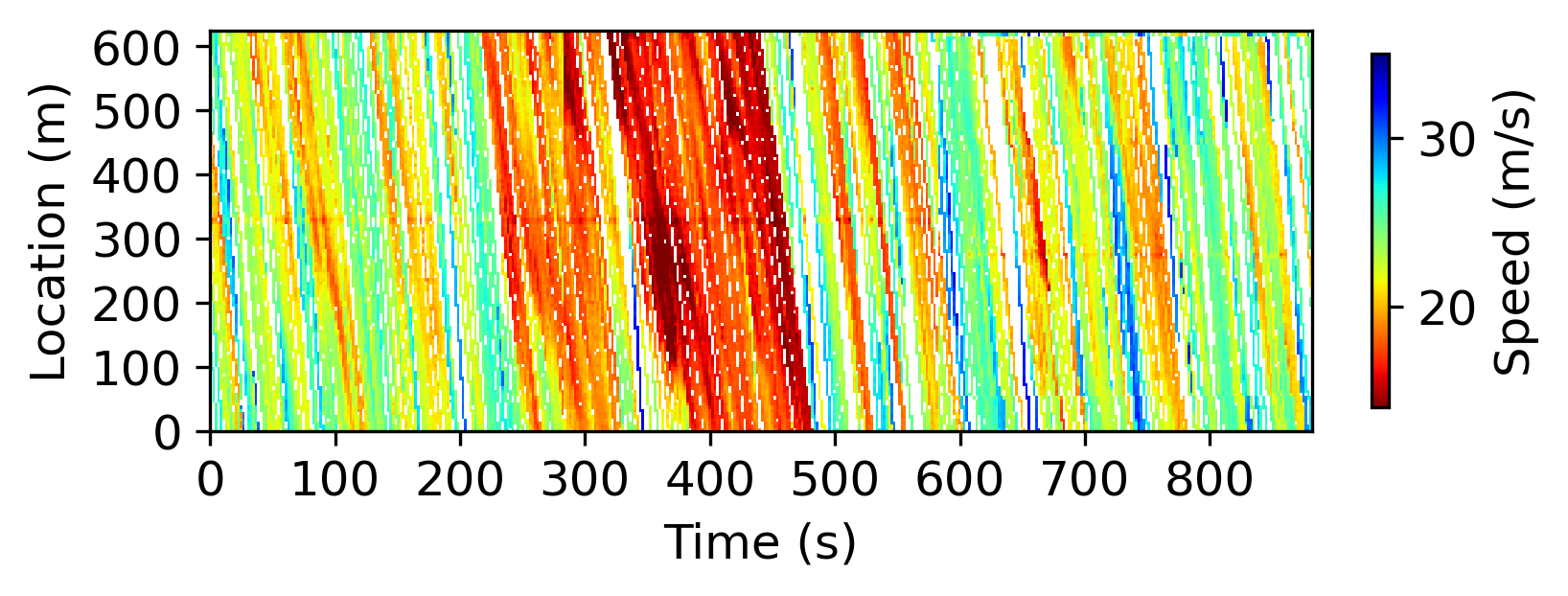}\label{speed_matrix_full_lane2}
}
\subfigure[Incomplete speed field (lane \#2).]{
\centering
\includegraphics[width = 0.31\textwidth]{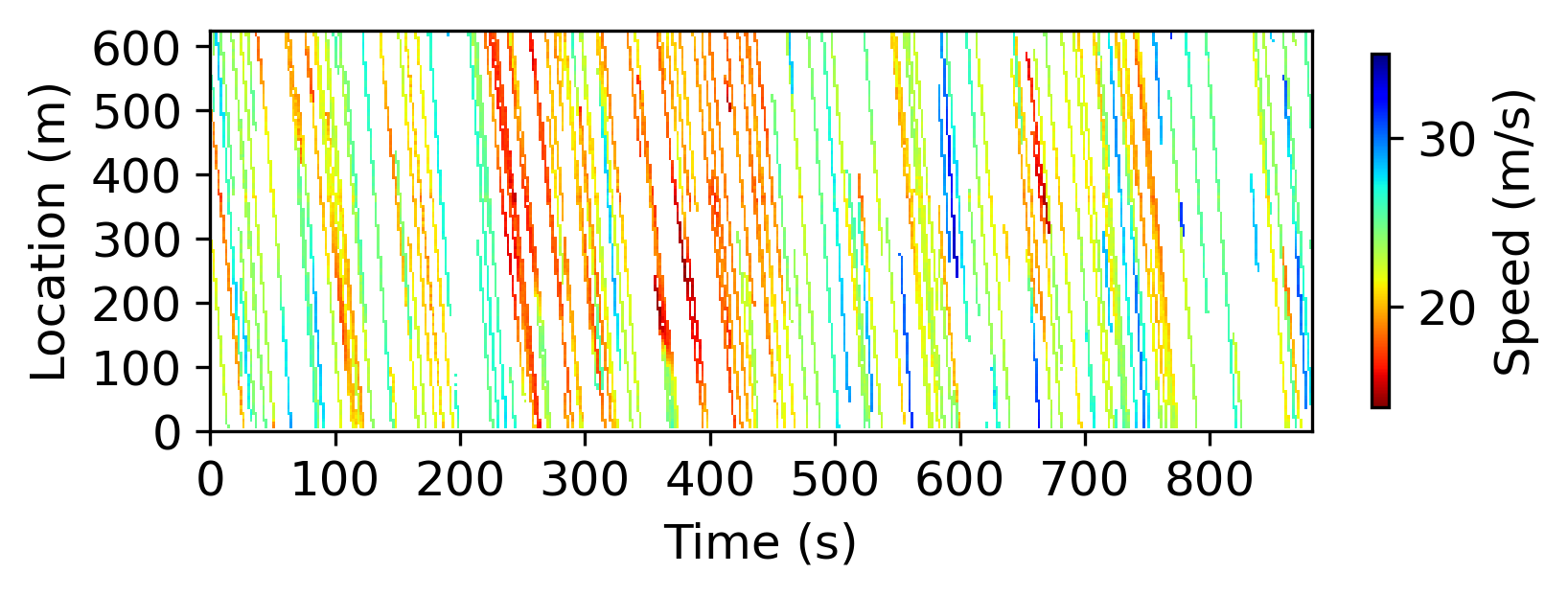}\label{speed_matrix_70_lane2}
}
\subfigure[Reconstructed speed field (lane \#2).]{
\centering
\includegraphics[width = 0.31\textwidth]{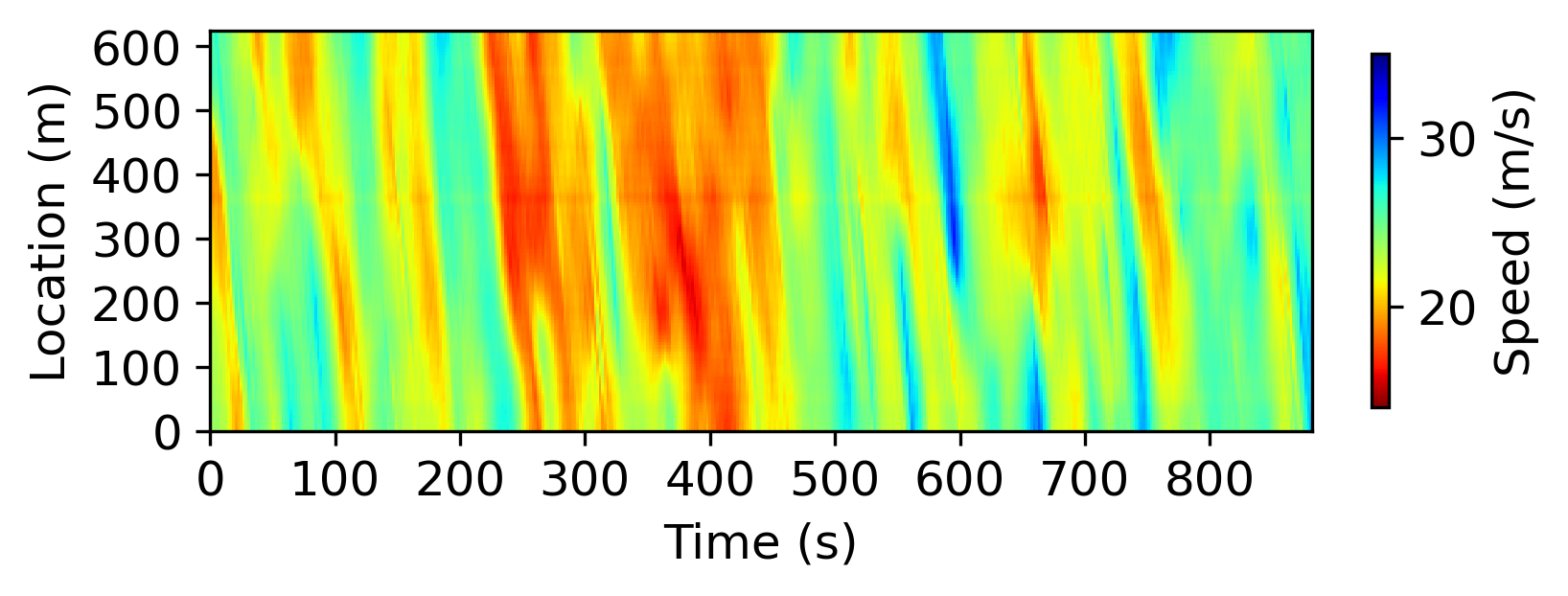}\label{CitySim_speed_field_70_LCR_2D_rec_lane2}
}
\subfigure[Original speed field (lane \#3).]{
\centering
\includegraphics[width = 0.31\textwidth]{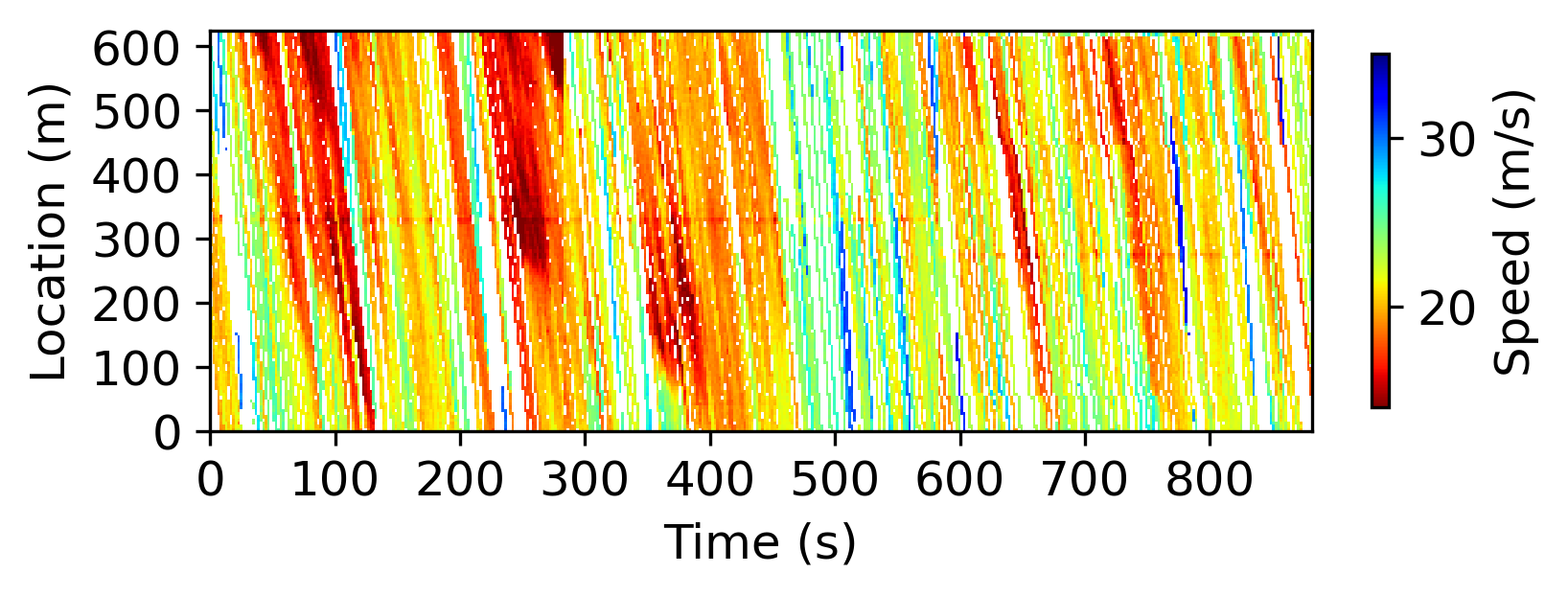}\label{speed_matrix_full_lane3}
}
\subfigure[Incomplete speed field (lane \#3).]{
\centering
\includegraphics[width = 0.31\textwidth]{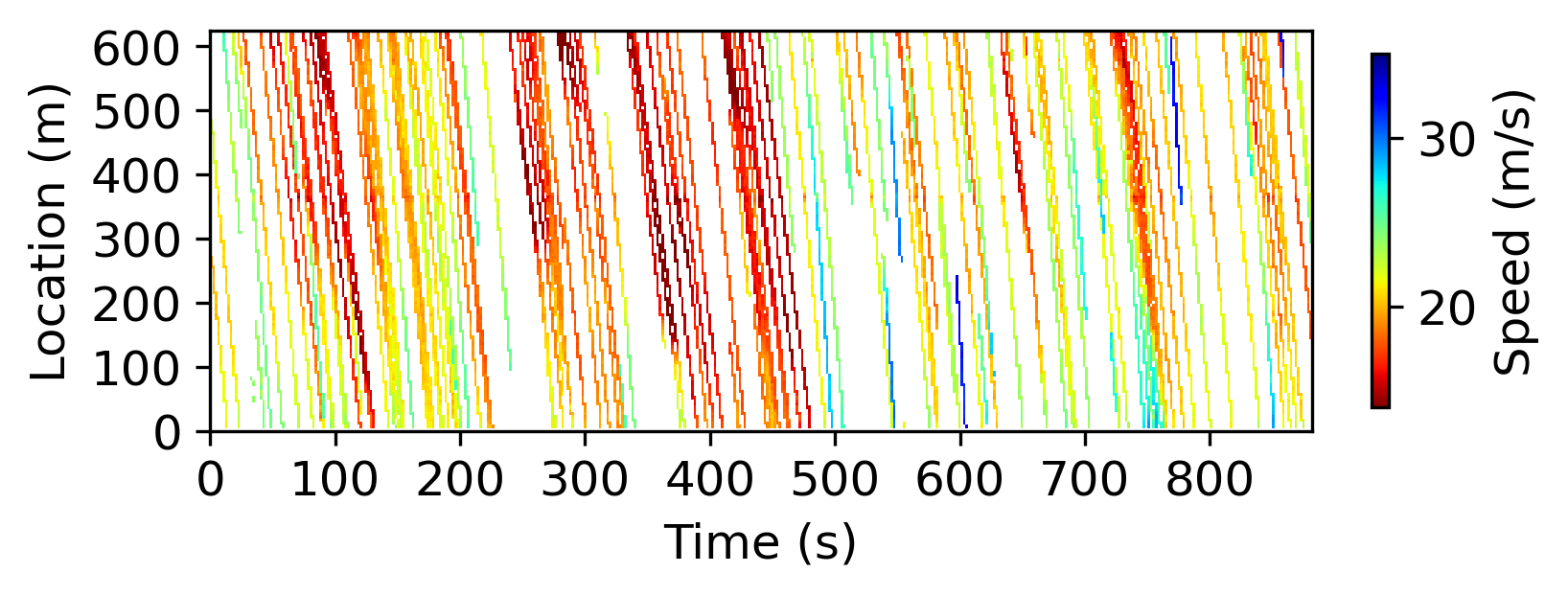}\label{speed_matrix_70_lane3}
}
\subfigure[Reconstructed speed field (lane \#3).]{
\centering
\includegraphics[width = 0.31\textwidth]{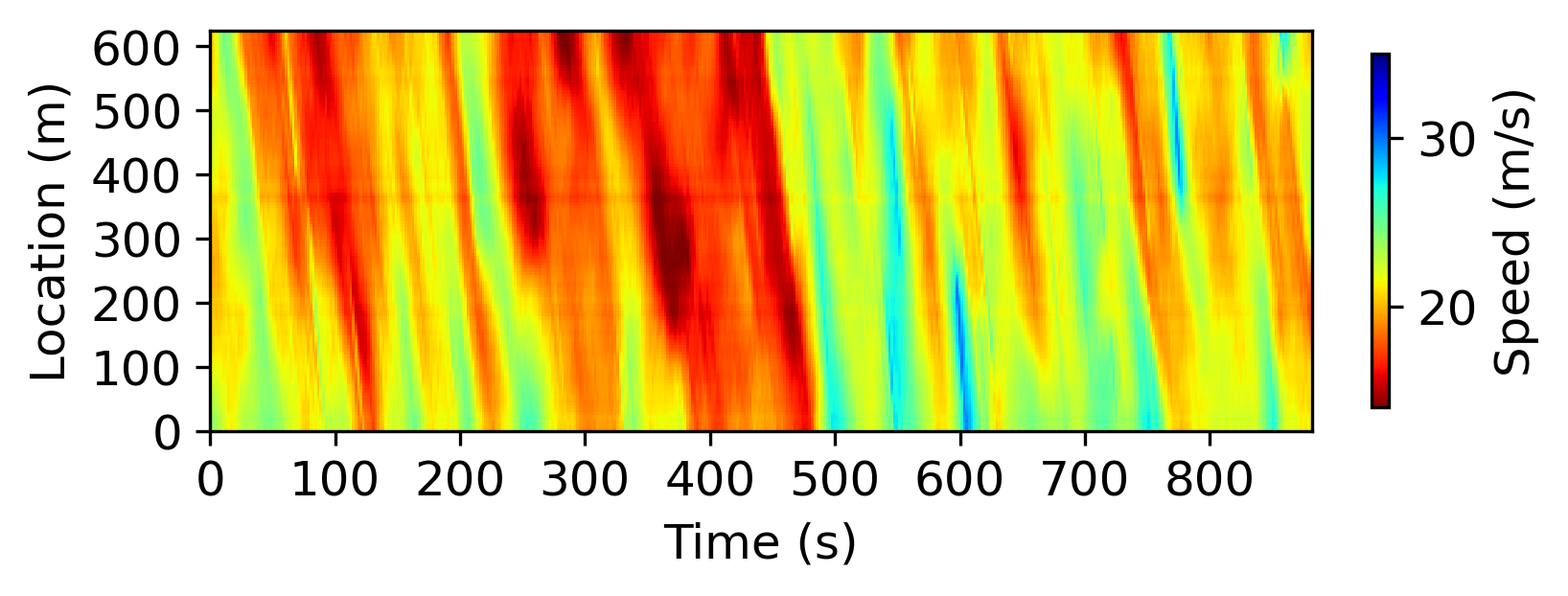}\label{CitySim_speed_field_70_LCR_2D_rec_lane3}
}
\caption{Speed field reconstruction achieved by LCR-2D on 70\% masked trajectories of the CitySim data. The speed fields with duration of 884 seconds are collected from the road segment of length 630 meters.}
\label{CitySim_speed_field_reconstruction_70}
\end{figure*}

\subsection{PeMS Traffic Speed Imputation}

In what follows, we study the generalization of LCR to high-dimensional data and evaluate the model on a large-scale traffic flow dataset. The data are collected by the California department of transportation through their Performance Measurement System (PeMS) \cite{chen2001freeway}. This dataset contains freeway traffic speed collected from 11,160 traffic measurement sensors over 4 weeks (the first 4 weeks in the year of 2018) with a 5-minute time resolution (288 time intervals per day) in California, USA.\footnote{The dataset is available at \url{https://doi.org/10.5281/zenodo.3939792}.} It can be arranged in a matrix of size $11160\times 8064$, and this dataset contains about 90 million observations.

To set up the imputation task, we randomly mask 30\%, 50\%, 70\%, and 90\% traffic speed observations as missing values, referred to as 30\%, 50\%, 70\%, and 90\% missing rates, respectively. To assess the imputation performance, we use the actual values of the masked missing entries as the ground truth to compute MAPE and RMSE. For comparison, the chosen baseline models are CircNNM \cite{liu2022recovery}, LRMC \cite{cai2010singular}, HaLRTC \cite{liu2013tensor}, LRTC-TNN \cite{chen2020nonconvex}, and nonstationary temporal matrix factorization (NoTMF \cite{chen2022nonstationary}). We also consider the comparison with (i) LCR-2D, (ii) LCR$_N$ (i.e., implementing LCR over $N$ univariate time series independently), (iii) LCR (i.e., the LCR model via vectorization on the $N$-by-$T$ multivariate time series), and (iv) CTNNM. Due to the day-to-day cyclical patterns of this dataset, we do not consider the flipping operation (see Fig.~\ref{flip_matrix}) in the numerical evaluation.

\begin{table*}[ht!]
\caption{Imputation performance (MAPE (\%)/RMSE) on the PeMS-4W traffic speed dataset. Note that the best results are emphasized in bold fonts, and we underline the second-best results.}
\label{pems_imputation}
\centering
\begin{tabular}{l|cccccccccc}
\toprule
Rate & LCR-2D & LCR$_N$ & LCR & CTNNM & CircNNM \cite{liu2022recovery} & LRMC \cite{cai2010singular} & HaLRTC \cite{liu2013tensor} & LRTC-TNN \cite{chen2020nonconvex} & NoTMF \cite{chen2022nonstationary} \\
\midrule
30\% & \underline{1.50}/\textbf{1.49} & \textbf{1.48}/\underline{1.50} & \underline{1.50}/\textbf{1.49} & 2.26/1.84 & 2.26/1.84 & 2.04/1.80 & 1.98/1.73 & 1.68/1.55 & 2.95/2.65 \\
50\% & \underline{1.76}/\textbf{1.69} & \textbf{1.73}/\underline{1.73} & \underline{1.76}/\textbf{1.69} & 2.67/2.14 & 2.69/2.15 & 2.43/2.12 & 2.22/1.98 & 1.93/1.77 & 3.05/2.73 \\
70\% & \textbf{2.07}/\textbf{2.06} & \textbf{2.07}/2.12 & \underline{2.08}/\underline{2.07} & 3.40/2.66 & 3.43/2.67 & 3.08/2.66 & 2.84/2.49 & 2.33/2.14 & 3.33/2.97 \\
90\% & \textbf{3.19}/\textbf{3.05} & 3.24/3.22 & \underline{3.21}/\underline{3.06} & 5.22/3.90 & 5.34/3.96 & 6.05/4.43 & 4.39/3.66 & 3.40/3.10 & 5.22/4.71 \\
\bottomrule
\end{tabular}
\end{table*}

As shown in Table~\ref{pems_imputation}, LCR-2D, LCR$_N$, and LCR achieve very competitive imputation accuracy and outperform baseline models. Among the baseline models in Table~\ref{pems_imputation}, CTNNM and CircNNM are the special cases of LCR-2D and LCR, respectively. Both models can be implemented by FFT like our LCR models. Comparing the imputation performance of CTNNM (or CircNNM) and LCR-2D (or LCR) shows the significant improvement of imputation achieved by LCR-2D over CTNNM, mainly due to the existence of the regularization term with Laplacian kernels. Therefore, introducing local trend modeling with Laplacian kernels in our LCR models is of great significance for traffic time series imputation. Despite the aforementioned comparison, our LCR models perform significantly better than some matrix/tensor completion algorithms such as LRMC, HaLRTC, and LRTC-TNN. Here, these matrix/tensor completion algorithms provide well-suited frameworks for reconstructing missing values in data matrix/tensor. However, they involve high time complexity in the singular value thresholding process, e.g., $$\mathcal{O}(\min\{N^2T,NT^2\})~~\text{(SVD)}\quad\text{vs.}\quad\mathcal{O}(NT\log(NT))~~\text{(FFT)},$$
making it extremely costly for large-scale problems. In contrast, our LCR models take an $\ell_1$-norm minimization in the frequency domain via the use of FFT. NoTMF jointly characterizes global and local trends by a unified and efficient temporal matrix factorization framework, but as can be seen, it is inferior to the LCR models.

\section{Conclusion}\label{conclusion}

In this study, we focus on reconstructing spatiotemporal traffic data from partial observations. To model the local trends in traffic time series, we introduce a Laplacian kernel for temporal regularization in the form of circular convolution. Following that definition, we propose an LCR model that integrates the temporal regularization into a circulant-matrix-based low-rank model for characterizing both global and local trends in traffic time series, bridging the gap between low-rank models and graph Laplacian models. When developing the solution algorithm, we borrow the properties of circulant matrix and circular convolution, and prove that our LCR model has a fast implementation with FFT. Specifically, the nuclear norm minimization with Laplacian kernelized temporal regularization can be converted into an $\ell_1$-norm minimization in complex space. Beyond univariate time series imputation, LCR can be easily adapted to multivariate or even multidimensional time series imputation.


In the numerical experiments, we conduct both univariate and multivariate time series imputation tasks on several real-world traffic flow datasets. On the sparse and noisy traffic data, LCR can accurately reconstruct traffic time series with the elimination of data noises and reinforce local time series trends, demonstrating the importance of local trend modeling. On the speed field reconstruction task, the results demonstrate the importance of spatiotemporal modeling with Laplacian kernels. On the large-scale dataset, LCR outperforms the baseline models and demonstrates strong generalization to high-dimensional problems due to the efficient implementation and relatively low time complexity. Despite the great success of temporal modeling in LCR, the key idea lies in connecting the low-rank models and the Laplacian kernelized regularization through FFT, which is also well-suited to some complicated spatiotemporal reconstruction problems. This study provides insight into traffic time series data modeling, nevertheless, the essential idea of LCR also matches the need for time series imputation and forecasting in other domains, while the existing studies (e.g., \cite{liu2022recovery, liu2022time}) already discussed the applications of circulant/convolution matrix-based methods to various time series data.



%



\section*{Acknowledgment}

Xinyu Chen would like to thank the Institute for Data Valorisation (IVADO) and the Interuniversity Research Centre on Enterprise Networks, Logistics and Transportation (CIRRELT) for providing the PhD Excellence Scholarship to support this study. The work of HanQin Cai is partially supported by NSF DMS 2304489.

\ifCLASSOPTIONcaptionsoff
  \newpage
\fi



%
\bibliographystyle{IEEEtran}
\bibliography{references}

%

\begin{IEEEbiography}[{\includegraphics[width=1in,height=1.25in,clip,keepaspectratio]{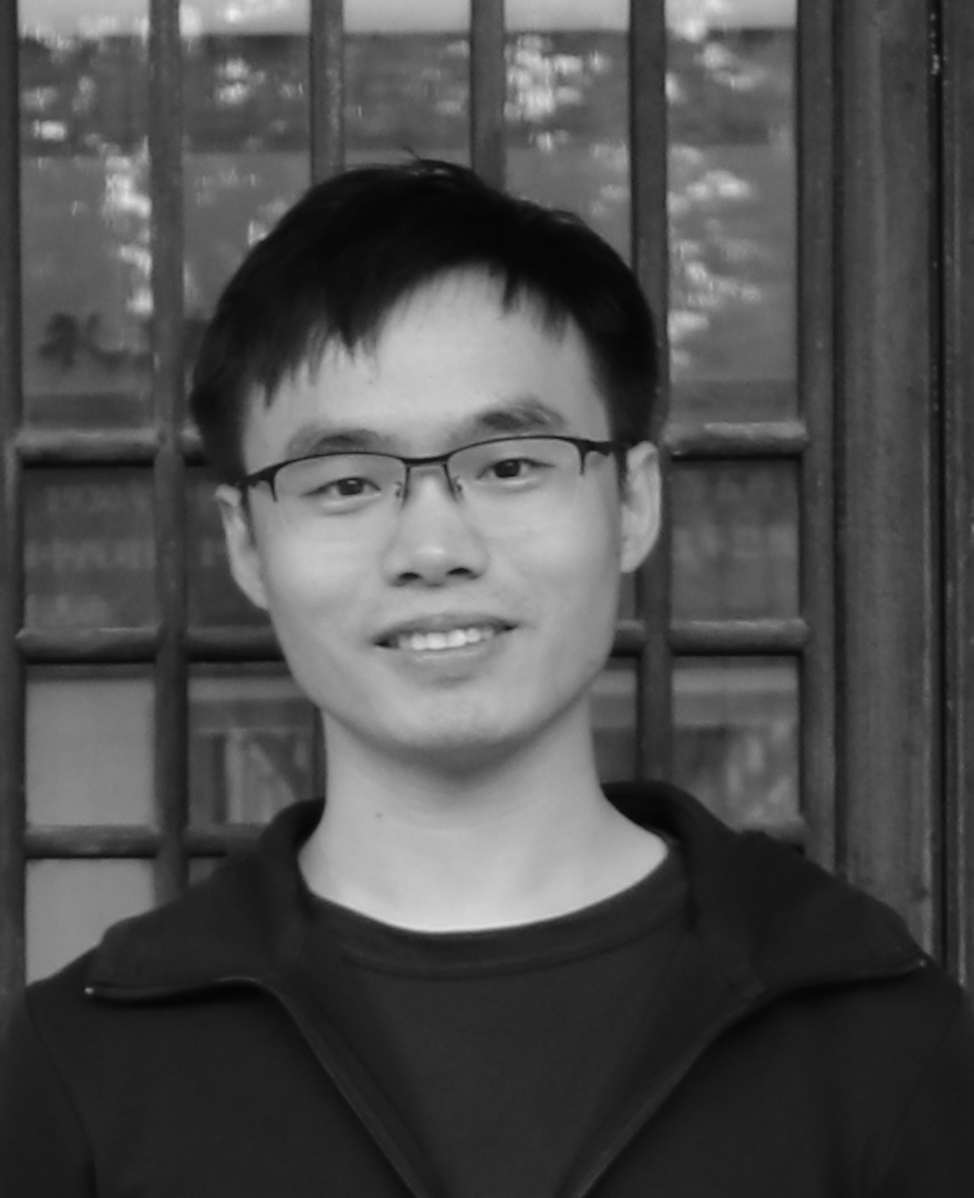}}]{Xinyu Chen} received his Ph.D. degree from the University of Montreal, Montreal, QC, Canada. He is now a Postdoctoral Associate at Massachusetts Institute of Technology, Cambridge, MA, United States. His current research centers on machine learning, spatiotemporal data modeling, intelligent transportation systems, and urban science.
\end{IEEEbiography}

\begin{IEEEbiography}[{\includegraphics[width=1in,height=1.25in,clip,keepaspectratio]{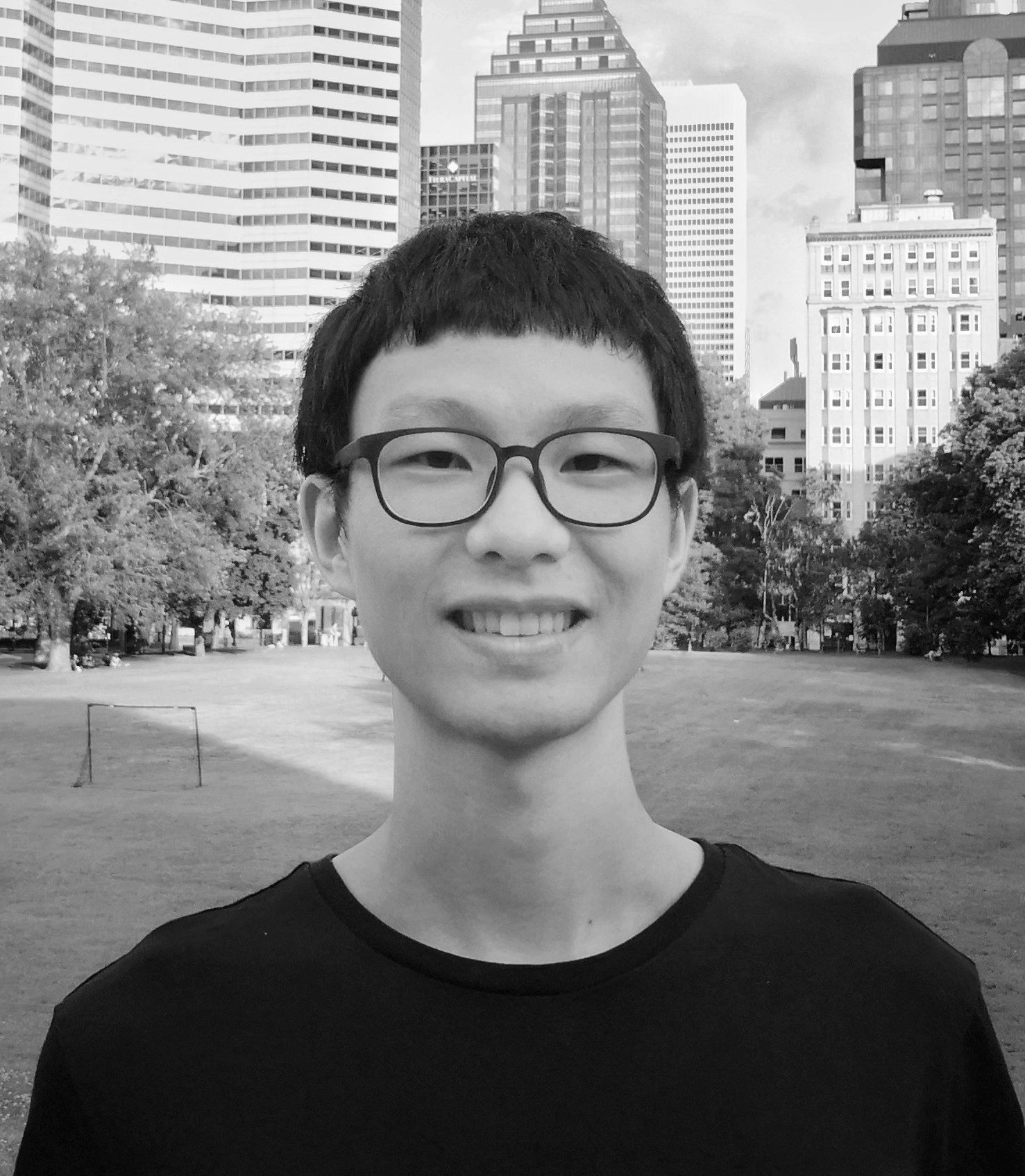}}]{Zhanhong Cheng} received his Ph.D. degree from McGill University. He received his B.S. and M.S. degrees from Harbin Institute of Technology, Harbin, China. He is now a Postdoc researcher in the Department of Civil Engineering at McGill University, Montreal, QC, Canada. His research interests include public transportation, travel behavior modeling, spatiotemporal forecasting, and machine learning in transportation.
\end{IEEEbiography}

\begin{IEEEbiography}[{\includegraphics[width=1in,height=1.25in,clip,keepaspectratio]{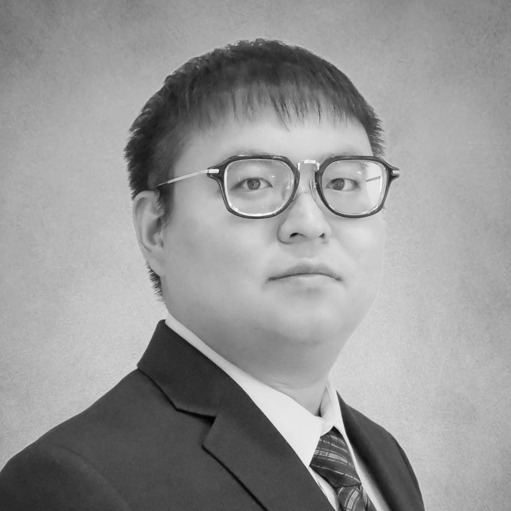}}]{HanQin Cai} received the PhD degree in applied mathematics and computational sciences from the University of Iowa. He is currently the Paul N. Somerville Endowed assistant professor with the Department of Statistics and Data Science and the Department of Computer Science, University of Central Florida. He is also the director of Data Science Lab. His research interests include machine learning, data science, mathematical optimization, and applied harmonic analysis.
\end{IEEEbiography}


\begin{IEEEbiography}[{\includegraphics[width=1in,height=1.25in,clip,keepaspectratio]{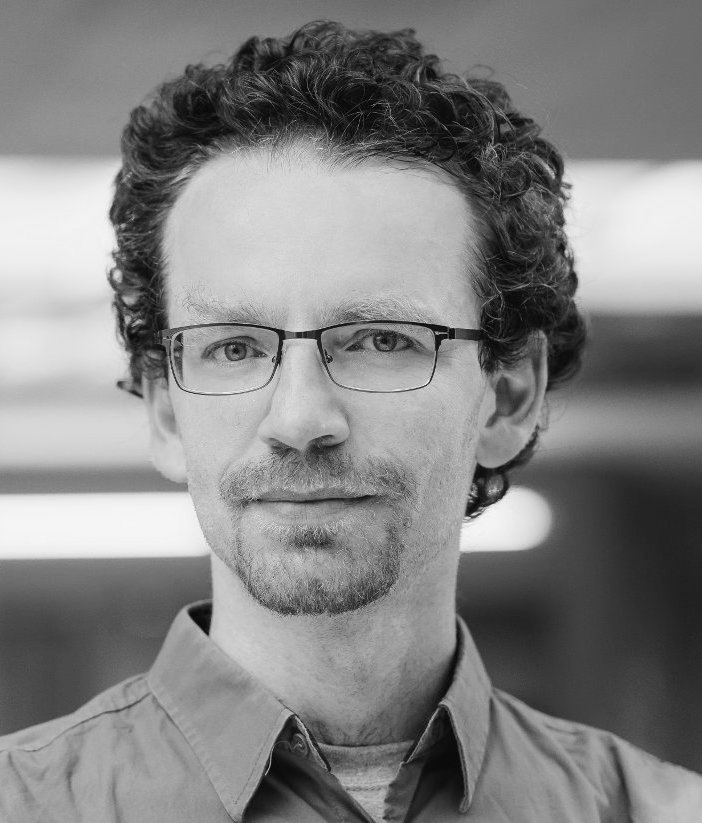}}]{Nicolas Saunier} received an engineering degree and a Doctorate (Ph.D.) in computer science from Telecom ParisTech, Paris, France, respectively in 2001 and 2005. He is currently a Full Professor with the Civil, Geological and Mining Engineering Department at Polytechnique Montreal, Montreal, QC, Canada. His research interests include intelligent transportation, road safety, and data science for transportation.
\end{IEEEbiography}

\begin{IEEEbiography}[{\includegraphics[width=1in,height=1.25in,clip,keepaspectratio]{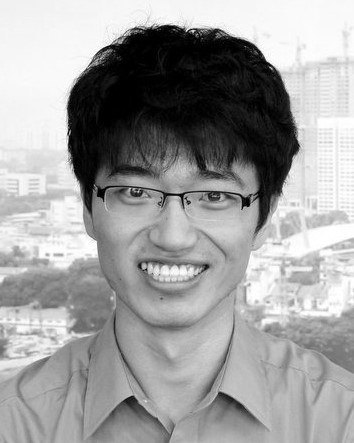}}]{Lijun Sun}(Senior Member, IEEE) received the B.S. degree in civil engineering from Tsinghua University, Beijing, China, in 2011, and the Ph.D. degree in civil engineering (transportation) from the National the University of Singapore in 2015. He is currently an Associate Professor and William Dawson Scholar in the Department of Civil Engineering, McGill University, Montreal, QC, Canada. His research centers on intelligent transportation systems, traffic control and management, spatiotemporal modeling, Bayesian statistics, and agent-based simulation.
\end{IEEEbiography}




\end{document}